\documentclass{article}
\usepackage[main,final]{neurips_2026}
\usepackage{natbib}

\setcitestyle{authoryear,round}

\usepackage[utf8]{inputenc} 
\usepackage[T1]{fontenc}    
\usepackage{hyperref}       
\usepackage{url}            
\usepackage{booktabs}       
\usepackage{amsfonts}       
\usepackage{nicefrac}       
\usepackage{microtype}      
\usepackage{comment}
\usepackage{multirow}
\usepackage{amsthm}
\usepackage{subcaption}
\usepackage{wrapfig}
\usepackage{tabularx}
\usepackage{duckuments}
\usepackage{caption}
\usepackage{graphicx}
\usepackage{amsmath}
\usepackage{algorithmic}
\usepackage{algorithm}
\usepackage{mathrsfs}

\usepackage{xspace}
\usepackage[table]{xcolor}

\usepackage[most]{tcolorbox}
\usepackage{listings}
\usepackage{xcolor}

\definecolor{boxbg}{gray}{0.96}
\definecolor{boxframe}{gray}{0.35}

\lstdefinestyle{mytt}{
  basicstyle=\ttfamily\footnotesize,
  columns=fullflexible,
  breaklines=true,
  showstringspaces=false,
  numbers=left,
  numberstyle=\tiny\color{gray},
  numbersep=6pt,
  xleftmargin=1.2em,
  frame=none
}

\newtcolorbox{promptbox}[1][]{
  enhanced,
  breakable,              
  colback=boxbg,
  colframe=boxframe,
  boxrule=0.8pt,
  arc=2mm,                
  left=2mm,right=2mm,top=1.5mm,bottom=1.5mm,
  #1
}

\newtheorem{theorem}{Theorem}[subsection]

\newtheorem{remark}[theorem]{Remark}
\newtheorem{proposition}[theorem]{Proposition}

\title{AEM: Adaptive Entropy Modulation for Multi-Turn Agentic Reinforcement Learning}
\author{
  Haotian Zhao$^1$\thanks{Equal contribution.} \\
  \And
  Songlin Zhou$^2$\footnotemark[1] \\
  \And
  Yuxin Zhang$^3$\footnotemark[1] \\
  \AND
  Stephen S.-T. Yau$^4$\\
  \And
  Wenyu Zhang$^3$ \\
  \And
  Lun Tian$^3$ \\
  \And
  Tianshu Zhu$^5$ \\
  \And
  Yifeng Huang$^6$ \\
  \And
  Yucheng Zeng$^3$ \\
  \And
  Jingnan Gu$^3$\thanks{Corresponding author.} \\
  \And
  Daxiang Dong$^3$\footnotemark[2] \\
  \And
  Jianmin Wu$^3$\footnotemark[2] \\
  \AND
  \normalfont
  $^1$\texttt{zhaohaotian021@gmail.com}, \quad Baidu\\
  $^2$\texttt{zhousl24@mails.tsinghua.edu.cn},\quad Tsinghua University\\
  $^3$\texttt{\{zhangyuxin15,zhangwenyu08, tianlun,
zengyucheng,}\\\texttt{gujingnan,dongdaxiang,wujianmin\}@baidu.com}, \quad Baidu\\
  $^4$\texttt{yau@uic.edu}, \quad Tsinghua University\\
  $^5$\texttt{tianshu.zhu@alumni.utoronto.ca}, \quad Baidu\\
  $^6$\texttt{yfhuang24@m.fudan.edu.cn},\quad Fudan University
}
\definecolor{brightgreen}{RGB}{0, 255, 0} 
\newcommand{\needrev}[1]{\textcolor{brightgreen}{#1}}
\newcommand{\name}{AEM\xspace}
\newcommand{\Hresp}{{\mathcal H}_{\mathrm{resp}}}
\newcommand{\E}{\mathbb E}
\newcommand{\Hpolicy}{{\mathcal H_{\mathrm{policy}}}}
\begin{document}

\maketitle

\begin{abstract}
Reinforcement learning (RL) has substantially improved the ability of large language model (LLM) agents to interact with environments and solve multi-turn tasks. However, effective agentic RL remains challenging: sparse outcome-only rewards provide limited guidance for assigning credit to individual steps within long interaction trajectories. Existing approaches often introduce dense intermediate supervision, such as process reward models or auxiliary self-supervised signals, which increases supervision and tuning complexity and may limit generalization across tasks and domains. 
We present \name, a supervision-free credit assignment method that adaptively modulates entropy dynamics during RL training to improve the exploration-exploitation trade-off. Since in agentic RL the environment is typically affected by a complete response, rather than an individual token, our analysis lifts entropy dynamics from the token level to the response level, aligning uncertainty estimation with the effective action granularity of LLM agents and reducing sensitivity to token-level sampling noise. We further show that entropy drift under natural-gradient updates is governed by the interaction between the sampled-response advantage and its relative surprisal. Motivated by this result, \name derives a practical response-level uncertainty proxy and uses it to rescale advantages, leveraging the evolving balance between positive and negative samples to naturally transition from exploration to exploitation. Extensive experiments on ALFWorld, WebShop, and SWE-bench-Verified with models ranging from 1.5B to 32B demonstrate that \name consistently improves strong RL baselines, including a +1.4\% gain when integrated into a state-of-the-art software-engineering RL training framework.



\end{abstract}

\section{Introduction}
Large language models (LLMs) are increasingly being deployed as interactive agents that solve complex tasks through multi-turn reasoning~\citep{xu2025amem,zeng2025reinforcing}, tool use~\mbox{\citep{shen2024taskbench,wu2024avatar}}, and sustained interaction with external environments~\citep{chentest,fang2025webevolver}. 
In such agentic settings, LLMs are no longer evaluated solely by isolated generation quality, but by their ability to make sequential decisions~\citep{zhang2025landscape}: repeatedly observing environment feedback, selecting actions, and refining their behavior across long interaction trajectories~\citep{shinn2023reflexion,erdoganplan}. 
This shift has enabled rapid progress in challenging domains such as autonomous software engineering~\citep{yang2024sweagent,yangswe}, embodied assistance~\citep{yang2024embodied,li2024embodied}, and GUI navigation~\citep{yuanse,limobileuse}.

Reinforcement learning (RL) has emerged as a central paradigm for improving such agents~\citep{dong2026agentic}, with group-based methods such as GRPO~\citep{shao2024deepseekmath} providing an effective value-free alternative to actor-critic training~\citep{konda1999actor,mnih2016a3c}. However, extending these methods from single-turn post-training to multi-turn agentic RL remains fundamentally challenging. Under such settings, feedback is sparse and outcome-based: the agent receives a reward only after completing a long trajectory~\citep{fenggroup}. As a result, different steps within the same trajectory often receive nearly indistinguishable learning signals, leading to ambiguous credit assignment and inefficient policy improvement.

Existing approaches address this issue by introducing denser credit signals. 
\textit{Reward shaping-based methods}, such as process reward models~\citep{lightman2023let}, provide dense step-level supervision but require additional models or annotations;
\textit{tree-structured optimization methods}, such as Tree-GRPO~\citep{ji2026tree} and ATPO~\citep{caoatpo}, enable fine-grained credit propagation via branching trajectories but incur high computational overhead in multi-turn settings;
\textit{self-supervised methods} (such as GiGPO~\citep{fenggroup} and IGPO~\citep{wang2025information}) infer step-level signals from trajectory structure without auxiliary supervision but are prone to context inconsistency, grouping bias, and heavy dependence on structural assumptions, which limit robustness and generalization.
Collectively, these limitations call for a scalable, fine-grained credit assignment framework that \textit{does not rely on extra supervision, heavy computation, and restrictive structural assumptions.}

Specifically, we notice that: (i) the policy's own entropy already provides an intrinsic signal for credit assignment: high-entropy responses typically reflect exploratory decisions, whereas low-entropy responses indicate more confident policy behavior; (ii) each completed \textit{response\footnote{In practice, a response usually combines reasoning and acting; in RL theory, it's the "action" sampled from the policy. To avoid ambiguity, we use the term "response".}} is the effective unit that changes the environment state. Therefore, we treat response-level entropy as an intrinsic signal for credit modulation. We demonstrate that the entropy drift induced by a sampled response is governed by the interaction between its advantage and relative response surprisal. This motivates \textbf{A}daptive \textbf{E}ntropy \textbf{M}odulation (\name), a credit assignment algorithm that uses a practical response-entropy proxy to rescale response-level advantages. \name adaptively preserves exploration early in training and promotes exploitation as successful samples become more prevalent, enhancing response diversity early in training while enabling more complete convergence in later stages.

Our contributions are three-fold.
\begin{itemize}
    \item We provide a response-level theoretical analysis of entropy dynamics in multi-turn agentic RL. By showing that entropy drift is determined by the interaction between sampled-response advantage and relative surprisal, our analysis reveals response-level uncertainty as a principled intrinsic signal for credit assignment.
    \item We propose \name, a \textit{supervision-free, lightweight, plug-in method} that modulates response-level advantages using an entropy-derived uncertainty proxy. By leveraging the evolving balance between positive and negative samples during training, \name adaptively guides the policy from early-stage exploration to late-stage exploitation.
    \item
    We conduct extensive experiments on ALFWorld, WebShop, and SWE-bench-Verified using models from 1.5B to 32B.  \name consistently improves multiple strong group-based RL baselines, with peak gains of \textbf{8.8\%} on GRPO with Qwen2.5-1.5B on ALFWorld, and a \textbf{+1.4\%} improvement when applied to DeepSWE on SWE-bench-Verified, demonstrating the effectiveness and generality of entropy-aware response-level credit modulation.
\end{itemize}

\section{Related Work}
\label{sec:related_work}

\paragraph{{From LLMs to Agentic RL.}}
Representative works such as ReAct~\citep{yao2023react} and Toolformer~\citep{schick2023toolformer} demonstrate that LLMs can interleave reasoning with actions and external tool calls, shifting the role of LLMs from passive generators to interactive decision-makers.
Training such agents increasingly relies on RL, where group-based methods such as RLOO~\citep{ahmadian2024back} and GRPO~\citep{shao2024deepseekmath} {have} emerged as a dominant approach.
Extending these methods from single-turn to multi-turn agentic settings exacerbates sparse rewards: feedback arrives only at the end, providing little guidance for intermediate decisions. The lack of step-level supervision yields high-variance gradients and ambiguous credit assignment, obscuring which intermediate actions should be reinforced or discouraged.



\paragraph{{Credit Assignment} in Agentic RL.}
Credit assignment is a long-standing challenge in agentic RL with delayed and sparse rewards. Existing efforts for step-level credit assignment in agentic RL differ mainly in where and how credit signals are derived. Some rely on \emph{external signals}, such as value functions or step-level supervision~\citep{schulman2017ppo,lightman2023let}, but introduce additional modeling and scaling overhead. Others derive credit \emph{internally} from sampled trajectories~\citep{fenggroup,wang2025information}, avoiding auxiliary supervision; some methods infer credit implicitly from trajectory attributes, while others further refine credit through structured propagation~\citep{caoatpo,ji2026tree} or reward redistribution~\citep{wang2025spa}, improving credit granularity but often incurring additional computational cost in multi-turn settings.
To address these limitations, a more general, lightweight, and adaptive credit assignment method is needed.

\paragraph{Entropy-Aware Policy Optimization.}
Entropy has long been used in RL as a regularization term for promoting exploration~\citep{cui2025entropy,petrenko2026entropy,chen2026flexible} and improving training stability~\citep{pmlr-v48-mniha16}. Recent studies have investigated entropy-aware training objectives, including entropy-regularized policy optimization~\citep{xu2025epo} and entropy-guided advantage scaling~\citep{wang2025harnessing,10.1145/3774904.3792301}. In addition, other work~\citep{shen2026on} have demonstrated that premature entropy collapse in the early phase of training can cause degraded downstream performance. Collectively, they indicate that policy entropy reflects model uncertainty and can provide an informative signal beyond external rewards.
Our method differs from prior entropy-aware approaches that either use entropy as a token-level auxiliary objective or regularizer, or leverage uncertainty for step-wise gradient recalibration. 
\name instead is motivated by a response-level analysis of entropy dynamics and uses response-level entropy only to rescale advantages, thereby adaptively shaping entropy dynamics throughout training.

\section{Theoretical Analysis}\label{sec:theoretical}



\subsection{Preliminaries}
\label{subsec:resp_entropy_geometry}

We consider a multi-turn agentic RL setting, where an agent policy \(\pi_\theta(\cdot\mid s)\) interacts with an environment over \(T\) steps. At each step \(t\in\{1,\dots,T\}\), the agent observes a state \(s_t\in\mathcal S\) (e.g., language messages, tool outputs, or webpage snapshots) and produces a textual response \(a_t\in\mathcal A_t \subset\mathcal V^{\le n}\) (e.g., free-form text, tool call with arguments, or interface selection), where \(\mathcal V\) is the LLM vocabulary and \(n\) is the maximum output length. Given prompt $s_0$, an episode yields a trajectory
\(
\tau=\{(s_0,a_0),\dots,(s_{T-1},a_{T-1})\},
\)
 sampled from $P_\theta(\cdot \mid s_0) = \prod_{t=0}^{T-1}\pi_\theta(\cdot \mid s_t)$ under Markov Decision Process assumption, conditioned on $s_0$. The policy is trained to maximize the expected trajectory return
$J(\theta)=\mathbb E_{\tau\sim P_\theta}[R(\tau)]$.
Each sampled response \(a_t\) at state \(s_t\) is associated with an advantage \(A(a_t,s_t)\) determined by the base advantage estimator. Hence, conditioning on a sampled pair \((a_t,s_t)\), the corresponding policy optimization surrogate objective is
\begin{equation}
    \ell_{a_t}(\pi) := A(a_t,s_t)\log \pi_\theta(a_t\mid s_t).
\end{equation}
In agentic RL, the environment typically reacts after a complete response is generated, making the response an effective interaction unit, rather than an individual token. The objective \(\ell_{a_t}(\pi)\) is consistent with this granularity, assigning a single learning signal to the whole response. Accordingly, we study response-level uncertainty, and define the response surprisal
\begin{equation}
    S(a_t\mid s_t):=-\log \pi_\theta(a_t\mid s_t)
=
-\sum_{\ell=1}^{|a_t|}\log p_\theta(y_\ell\mid s_t,y_{<\ell}),
\end{equation}
with the response-level Shannon entropy
\begin{equation}
    \Hresp(s_t):=
-\sum_{a_t \in \mathcal A_t} \pi_\theta(a_t\mid s_t)\log \pi_\theta(a_t\mid s_t) = \mathbb E_{a_t \sim \pi_\theta(\cdot|s_t)}[S(a_t\mid s_t)].
\end{equation}

\subsection{Response-Level Entropy Geometry}

\begin{theorem}[Relationship among token, response, and policy entropy. Proved in Appendix~\ref{appendix:response}]\label{thm:responselevelentropy}
Let $a_t = (Y_1,..,Y_L) \sim \pi_\theta(\cdot|s_t)$ denote a sampled response spanned by tokens $Y_\ell \sim p_\theta(\cdot\mid s_t, Y_{<\ell})$, and $s_0$ denote the initial state in the dataset $\mathcal D$. The token-level entropy $\mathcal H_\ell(a_t,s_t)$ and the policy entropy $\mathcal H_{\mathrm{policy}}$ are respectively formulated by
\begin{equation}
   \mathcal H_\ell(a_t,s_t):=\E_{Y_\ell \sim p_\theta(\cdot\mid s_t, Y_{<\ell})}[-\log p_\theta(Y_\ell|s_t,Y_{<\ell})] =-\sum_{y\in \mathcal{V}}p_\theta(y|s_t,y_{<\ell})\log p_\theta(y|s_t,y_{<\ell});
\end{equation}
\begin{equation}
    \Hpolicy = \mathbb E_{s_0 \sim \mathcal D, \tau \sim P_\theta(\cdot \mid s_0)}\left[\sum_{t=0}^{T-1}\sum_{\ell=1}^{|a_t|}\mathcal H_\ell(a_t,s_t)\right].
\end{equation}
Then, the response-level entropy is the expectation of token-level entropy sum:
\begin{align}
    \Hresp(s_t) = \mathbb E_{a_t\sim \pi_\theta(\cdot|s_t)}\left[\sum_{\ell\ge 1}\mathcal H_\ell(a_t,s_t)\mathbf 1\{\ell \le |a_t|\}|s_t\right],
\end{align}
and the policy entropy is the expectation of response-level entropy sum:
\begin{align}\label{eq:policy}
    \Hpolicy = \mathbb E_{s_0 \sim \mathcal D, \tau \sim P_\theta(\cdot\mid s_0)} \left[\sum_{t=0}^{T-1} \Hresp(s_t)\right].
\end{align}
Therefore, response-level entropy provides a structurally faithful intermediate uncertainty measure: entropy modulation applied at the response level induces corresponding changes in policy entropy, while being less sensitive to token-level sampling variation.
\end{theorem}

To analyze how a sampled response and its advantage reshape the policy distribution from an information-theoretic perspective, we formulate the policy given state $s$ on the probability simplex $\Delta^\circ(\mathcal A_s)$ equipped with the Fisher-Rao metric (\citep{amari2000methods}, \citep{nielsen2020elementary}), this canonical information metric is the local quadratic form of KL divergence (Details in Appendix~\ref{appendix:simplex}). 
Within this geometry, the natural gradient~\citep{kakade2001natural} induces parameterization-invariant policy updates. By analyzing response-level entropy dynamics and aggregating them over visited states, the following theorem shows that the entropy dynamics \(\Hpolicy\) is governed by the advantage and relative surprisal of sampled responses.
\begin{figure}[t]
    \centering
    \includegraphics[width=0.7\linewidth]{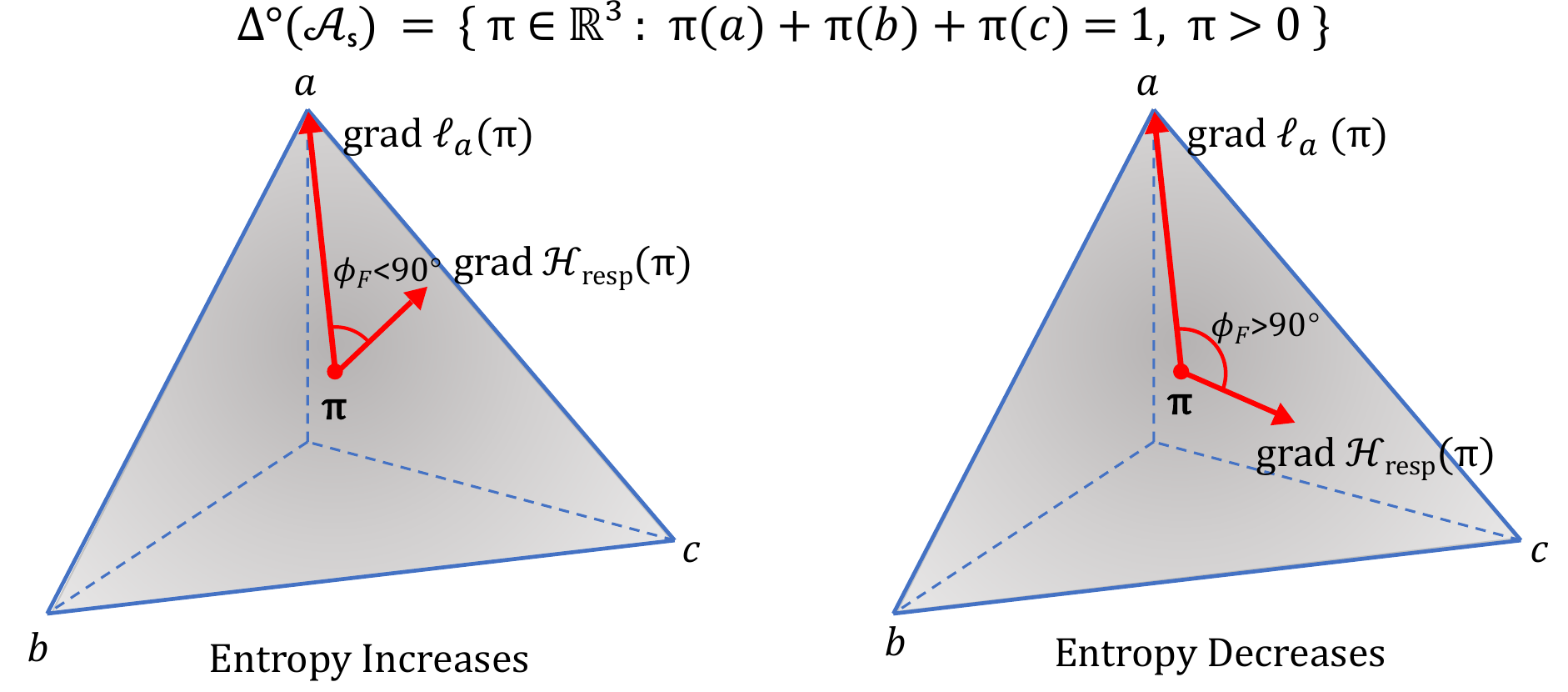}
    \caption{An example on a three-response policy simplex: here $\phi_F$ denotes the Fisher-Rao angle: $\phi_F = \angle_F(\operatorname{grad}^F\ell_a,\operatorname{grad}^F \mathcal \Hresp)$. Entropy increases along the training direction when $D_{\mathrm{RL}}(a;s) > 0$ i.e., $\phi_F < 90^\circ$, and decreases otherwise.}

    \label{fig:simplex-entropy-drift}
    \vspace{-1.0em}
\end{figure}
\begin{theorem}[Entropy drift under fixed occupancy. Proved in Appendix~\ref{appendix:fisher}]
\label{thm:fisher}
Let $\operatorname{grad}^F$ denote the natural gradient on the policy simplex $\Delta^\circ(\mathcal A_s)$, then the directional derivative of $\Hresp$ along the update direction $\operatorname{grad}^F \ell_a(\pi)$ satisfies
\begin{align}
D_{\mathrm{RL}}^{\mathrm{resp}}(a;s)
&:=
\left\langle
\operatorname{grad}^F \Hresp(\pi),\,
\operatorname{grad}^F \ell_a(\pi)
\right\rangle_{\mathrm{Fisher\text{-}Rao}}
=
A(a,s)\bigl(S(a\mid s)-\Hresp(s)\bigr).
\label{eq:thm:fisher1}
\end{align}
Assume a local policy update under a frozen rollout distribution, i.e., when differentiating the policy entropy objective, we do not propagate gradients through the rollout distribution $P_\theta$. Then the policy entropy drift induced by a sampled response $a$ equals the visitation-weighted expectation of the response-level entropy drift:
\begin{align}
D_{\mathrm{RL}}(a;s)
&:=
\left\langle
\operatorname{grad}^F \Hpolicy(\pi),\,
\operatorname{grad}^F \ell_a(\pi)
\right\rangle_{\mathrm{Fisher\text{-}Rao}}\notag\\
&=
\sum_{t=0}^{T-1} \mathbb P_{s_0 \sim \mathcal D, \tau \sim P_\theta}\left[s_t = s\right]\,A(a,s)\bigl(S(a\mid s)-\Hresp(s)\bigr).
\label{eq:thm:fisher2}
\end{align}
Therefore, the entropy dynamics in training is determined by advantage of sampled response $A(a,s)$ and relative surprisal $S(a\mid s) - \Hresp$ (see Figure~\ref{fig:simplex-entropy-drift}):
\begin{align}\label{eq:dynamics}
    &\operatorname{sgn}(A(a,s)(S(a\mid s) - \Hresp)) > 0 \implies \text{entropy increases}; \notag\\
    &\operatorname{sgn}(A(a,s)(S(a\mid s) - \Hresp)) < 0 \implies \text{entropy decreases}.
\end{align}
\end{theorem}
\begin{remark}
    In some practical agentic RL settings, the objective is not purely reward-driven: many methods also include entropy regularization or KL penalties. In Appendix~\ref{appendix:fisher}, we extend the theorem to the regularized objective:
\begin{equation}
    \ell_a(\pi_\theta) = A(a,s)\log\pi_\theta + \beta \psi(\Hresp(\pi_\theta)) - \gamma D_{KL}(\pi_\theta\| \pi_{ref}),
\end{equation}
where $\psi$ is a positive increasing function and $\beta,\gamma$ are regularization coefficients.

It is demonstrated that, since these regularization terms act at the state level, they do not change the response-dependent modulation principle implemented by \name.
\end{remark}

Theorem~\ref{thm:fisher} shows that the entropy drift induced by a sampled response is governed by the interaction between its advantage and relative surprisal. This provides a theoretical basis for modulating entropy dynamics through response-level credit signals: by rescaling response advantages according to relative surprisal, one can induce entropy-increasing or entropy-decreasing pressure without changing the underlying RL optimization backbone. This mechanism is intrinsic to policy space and independent of any specific neural parameterization; Appendix~\ref{appendix:parametrized_version} presents its parameter-space counterpart. Motivated by this observation, we next introduce AEM.

\section{AEM: Adaptive Entropy Modulation}
\subsection{What is \name?}
\name{} is a plug-in response-level advantage modulation method applied on top of a base advantage estimator. It leverages a proxy of relative surprisal as an intrinsic signal to regulate entropy dynamics. Let \(A^{\mathrm{base}}_{i,t}\) denote the response-level advantage produced by the base estimator for the \(t\)-th turn in the \(i\)-th rollout \(\mathcal S_i\). Here, \(\mathcal S_i = \{S_{i,1}, \ldots, S_{i,K_i}\}\), where each span \(S_{i,t} = [\text{begin\_token}_{i,t}, \text{end\_token}_{i,t}]\) corresponds to one completed response generated before the next environment transition. For each environment-reactive response span \(S_{i,t}\), \name{} computes a scalar modulation coefficient \(\alpha_{i,t}\) and applies it uniformly to all tokens in the span:
\[
A^{\mathrm{AEM}}_{i,t}
=
\alpha_{i,t} A^{\mathrm{base}}_{i,t}.
\]
Thus, \name{} only rescales response-level advantages, inducing entropy-increasing pressure on negative responses and entropy-decreasing pressure on positive responses. As training progresses and the proportion of positive responses increases, this modulation naturally shifts the dominant entropy pressure from exploration-preserving to exploitation-promoting, enabling an adaptive transition from exploration to exploitation during RL training.


\subsection{Modulation Mechanism}
Since the state-specific baseline $\Hresp(s_t)$ is not directly tractable during training, AEM does not explicitly reconstruct the exact gap. Instead, it converts the relative magnitude within the group of this proxy into a modulation coefficient $\alpha$, so that $\alpha>1$ and $\alpha<1$ serve as practical indicators of lower- and higher-surprisal responses.

Given the $t$-th response in a rollout, Theorem~\ref{thm:fisher} shows that the sign of the local entropy drift is jointly governed by the relative surprisal $S(a_t\mid s_t)-\Hresp(s_t)$ and the response advantage $A(a_t,s_t)$. To reduce the sensitivity to the particular sampled tokens, we use the predictable proxy $\sum_{\ell=1}^{|a|}\mathcal H_\ell(a_t,s_t)$ for $S(a_t|s_t)$ from Doob's decomposition\footnote{For a response, Doob's decomposition writes its surprisal as the sum of a predictable component, i.e., the token-level conditional entropies, and a zero-mean martingale term capturing fluctuations. AEM uses the normalized predictable component as the proxy to reduce sensitivity to fluctuations while retaining the response-level uncertainty signal.} (see Appendix~\ref{appendix:doob} for details).

With a length normalization to make the response-level entropy scale-free, we consider 
\begin{equation}\label{eq:proxy}
\bar{\mathcal H}_{i,t}
=
\frac{1}{|S_{i,t}|}
\sum_{\ell\in S_{i,t}} \mathcal H_{\ell}(a_t,s_t),
\end{equation}
and apply a monotone decreasing map from 
$\bar{\mathcal H}_{i,t}$ to a response-uniform coefficient $\alpha_{i,t}$.

Let $\mathcal G$ be a group as the set of all responses in the trajectories generated by a prompt.
We normalize ${\bar{\mathcal H}}_{i,t}$ within group min-max scaling to avoid numerical explosion:
\begin{equation}
\tilde{\mathcal H}_{i,t}
=
\frac{
\bar{\mathcal H}_{i,t}-\min_{(j,n)\in\mathcal G}\bar{\mathcal H}_{j,n}
}
{
\max_{(j,n)\in\mathcal G}\bar{\mathcal H}_{j,n}
-
\min_{(j,n)\in\mathcal G}\bar{\mathcal H}_{j,n}
+\varepsilon}, \quad \text{for } (i,t) \in \mathcal G.
\end{equation}
When \(\max_{(j,n)\in\mathcal G}\bar{\mathcal H}_{j,n}-\min_{(j,n)\in\mathcal G}\bar{\mathcal H}_{j,n} < 0.1\), we set $\alpha_{i,t} = 1$ to avoid sampling noise. Otherwise, we define the self-calibrated modulation coefficient with temperature $\lambda$:
\begin{equation}
\alpha_{i,t}
=
\frac{
\exp(-\lambda \tilde{\mathcal H}_{i,t})
}{
\frac{1}{|\mathcal G|}\sum_{(j,n)\in\mathcal G} \exp(-\lambda \tilde{\mathcal H}_{j,n})+\varepsilon
},\quad  \text{for } (i,t) \in \mathcal G.
\end{equation}
Hence AEM relatively upweights ($\alpha>1$) spans with lower relative surprisal proxy within the group, and downweights ($\alpha<1$) those with higher relative surprisal proxy, while preserving the overall modulation scale through self-calibration. Ablation studies in Appendix~\ref{appendix:ablation} demonstrate the importance of correct direction of entropy-aware credit assignment and group normalization in AEM.

\subsection{Exploration-Exploitation Transition}
\hyperref[sec:consistency]{Analysis A} shows that $\alpha - 1$ has a strong correlation with $-(S - \Hresp)$, providing empirical support for the theoretical connection in Eq.~\eqref{eq:dynamics}. \hyperref[sec:entropy_illustration]{Analysis B} further demonstrates that $A(a,s)$ and $\alpha - 1$ indeed determine the practical entropy dynamics $\tilde D_{RL}^{\mathrm{base}}(a;s)$:
\begin{equation}\label{eq:practical}
    \operatorname{sgn}\tilde D_{RL}^{\mathrm{base}}(a;s) \approx -\operatorname{sgn}\left(A(a,s)(\alpha-1)\right).
\end{equation}

Generally, AEM systematically shifts the intrinsic entropy drift based purely on the sign of the advantage:\begin{equation}
    \operatorname{sgn}\bigl(\tilde D_{\mathrm{RL}}^{\mathrm{AEM}}-\tilde D_{\mathrm{RL}}^{\mathrm{base}}\bigr)
=
-\operatorname{sgn}\bigl((\alpha-1)^2A(a,s)\bigr) = -\operatorname{sgn}A(a,s).
\end{equation}
By Eq.~\eqref{eq:policy} in Theorem~\ref{thm:responselevelentropy}, through modulating the entropy drift of relatively many responses, AEM induces a corresponding shift in the policy entropy.
As training progresses, it naturally induces an implicit transition from exploration to exploitation:

\textbf{Exploration.} For negative responses $A(a,s)<0$ which are relatively prevalent in early stage of RL training, AEM provides entropy-increasing pressure:
\begin{equation}
\begin{cases}
\bar{\mathcal H}_{i,t}\ \text{relatively large}
    \implies
    \alpha_{i,t} < 1,\tilde D_{\mathrm{RL}}^{\mathrm{base}} < 0
    \implies
    \text{attenuate entropy-decreasing},\\[0.4em]
\bar{\mathcal H}_{i,t}\ \text{relatively small}
    \implies
    \alpha_{i,t} > 1,\tilde D_{\mathrm{RL}}^{\mathrm{base}} > 0
    \implies
    \text{amplify entropy-increasing}.
\end{cases}
\end{equation}

\textbf{Exploitation.} For positive responses $A(a,s)>0$ which are relatively prevalent in late stage of RL training, AEM provides entropy-decreasing pressure:
\begin{equation}
\begin{cases}
\bar{\mathcal H}_{i,t}\ \text{relatively large}
    \implies
    \alpha_{i,t} < 1,\tilde D_{\mathrm{RL}}^{\mathrm{base}} > 0
    \implies
    \text{attenuate entropy-increasing},\\[0.4em]
\bar{\mathcal H}_{i,t}\ \text{relatively small}
    \implies
    \alpha_{i,t} > 1,\tilde D_{\mathrm{RL}}^{\mathrm{base}} < 0
    \implies
    \text{amplify entropy-decreasing}.
\end{cases}
\end{equation}

\hyperref[sec:exploration]{Analysis C} shows that AEM mitigates early entropy collapse, promotes more complete late-stage convergence, and improves final performance.

\section{Experiments}
\label{sec:experiments}
Subsection~\ref{sec:setup} introduces the benchmarks and baseline methods. Appendix~\ref{appendix:implementation} shows the implementation details used in our experiments. Subsection~\ref{sec:performance} reports the empirical results of AEM when integrated with different baselines across benchmarks. Subsection~\ref{sec:analysis} analyzes the mechanism underlying AEM, including the consistency between its modulation coefficient and relative surprisal, the resulting entropy dynamics, and the induced exploration-exploitation transition during training. For all experiments in subsection~\ref{sec:analysis}, we use Qwen2.5-1.5B on WebShop, with GRPO as the base estimator.
Subsection~\ref{subsec:comp_cost} analyzes the computational cost of \name.
Appendix~\ref{sec:empg_comparison} provides empirical comparison of AEM against closely related entropy-aware method EMPG~\citep{wang2025harnessing}.
Finally, Appendix~\ref{appendix:ablation} presents ablation studies comparing AEM with several design variants.

\subsection{Setup}\label{sec:setup}
\paragraph{Benchmarks.}
We evaluate \name on three challenging multi-turn LLM agent benchmarks: ALFWorld~\citep{ALFWorld20}, WebShop~\citep{yao2022webshop}, and SWE-bench-Verified~\citep{jimenez2024swebench}. ALFWorld evaluates text-based embodied decision-making across six household task categories: Pick \& Place (Pick), Examine in Light (Look), Clean \& Place (Clean), Heat \& Place (Heat), Cool \& Place (Cool), and Pick Two \& Place (Pick2). WebShop evaluates web-based shopping agents in a simulated HTML environment with large-scale product search, navigation, and item selection. SWE-bench-Verified is a curated subset of SWE-bench with expert-validated tasks, stable environments, and verifiable solutions for evaluating software engineering agents.

\paragraph{Baselines.}
For ALFWorld and WebShop, we compare \name against several competitive baselines, including: (1) closed-source LLMs: GPT-5.2-Pro~\citep{gpt5-2} and Gemini-3-Pro~\citep{gemini3pro};
(2) prompting-based methods: ReAct~\citep{yao2023react}, which interleaves reasoning traces and executable actions to enable step-by-step decision-making in interactive environments;
(3) reinforcement learning methods: PPO~\citep{schulman2017ppo}, GRPO~\citep{shao2024deepseekmath},
DAPO~\citep{yu2025dapo},
GSPO~\citep{zheng2025group}. The algorithmic details of these baselines are shown in Appendix~\ref{appendix:baselines}. 
To further validate the generality of \name{} in complex agentic RL scenarios, we integrate it into DeepSWE~\citep{deepswe2025}, a state-of-the-art open-source RL framework for multi-turn software-engineering agents. DeepSWE adapts GRPO to SWE agent training with a GRPO++ recipe that improves long-horizon optimization via clip-higher, removal of KL and entropy losses, mitigation of difficulty and length biases, leave-one-out advantage estimation, and compact trajectory filtering.
Full implementation details are deferred to Appendix~\ref{appendix:implementation}.

\begin{table}[t]
\centering
\caption{Performance comparison on ALFWorld and WebShop benchmarks. The results of ReAct and PPO are adopted from~\citep{fenggroup}. }
\small
\renewcommand{\arraystretch}{1.0}
\setlength{\tabcolsep}{6.6pt}
\begin{tabularx}{\linewidth}{X ccccccc|cc}
\toprule
\multirow{2}{*}{Method} 
& \multicolumn{7}{c}{ALFWorld} 
& \multicolumn{2}{c}{WebShop} \\
\cmidrule(lr){2-8} \cmidrule(lr){9-10}
& Pick & Look & Clean & Heat & Cool & Pick2 & All & Score & Succ. (\%)\\
\midrule

\multicolumn{10}{c}{\textit{Closed-Source Model}} \\

GPT-5.2-Pro 
& 100 & 100 & 100 & 61.3 & 87.0 & 100 & 88.8 & 46.6
& 44.4 \\

Gemini-3-Pro 
& 100 & 100 & 96.8 & 100 & 100 & 100 & 99.3 & 60.8
& 56.7 \\

\midrule
\multicolumn{10}{c}{\textit{Qwen2.5-1.5B-Instruct}} \\


ReAct 
& 17.4 & 20.5 & 15.7 & 6.2 & 7.7 & 2.0 & 12.8 
& 40.1 & 11.3 \\

PPO
& 64.8 & 40.5 & 57.1
& 60.6 & 46.4 & 47.4 & 54.4$_{\pm3.1}$ 
& 73.8$_{\pm3.0}$ & 51.5$_{\pm2.9}$ \\

GRPO 
& 78.2 & 49.9 & 70.5
& 72.0 & 75.0 & 39.2 & 68.0$_{\pm0.8}$
& 83.6$_{\pm0.2}$ & 65.0$_{\pm0.6}$ \\

\cellcolor{blue!10}$\quad+$AEM 
& \cellcolor{blue!10}88.6
& \cellcolor{blue!10}67.6
& \cellcolor{blue!10}76.4
& \cellcolor{blue!10}60.9
& \cellcolor{blue!10}76.7
& \cellcolor{blue!10}69.9
& \cellcolor{blue!10}\textbf{76.8$_{\pm1.8}$} 
& \cellcolor{blue!10}\textbf{86.4$_{\pm2.1}$} 
& \cellcolor{blue!10}\textbf{70.6$_{\pm2.4}$}  \\

GSPO 
& 75.4 & 54.2 & 64.6 
& 70.0 & 74.6 & 30.0 & 66.7$_{\pm 5.3}$
& 75.1$_{\pm7.1}$ & 61.5$_{\pm4.5}$ \\

\cellcolor{blue!10}$\quad+$AEM 
& \cellcolor{blue!10}75.5 & \cellcolor{blue!10}56.5 & \cellcolor{blue!10}78.1 
& \cellcolor{blue!10}75.0 & \cellcolor{blue!10}70.2 & \cellcolor{blue!10}46.7 & \cellcolor{blue!10}\textbf{71.9$_{\pm 8.4}$}
& \cellcolor{blue!10}\textbf{76.3$_{\pm3.8}$} & \cellcolor{blue!10}\textbf{66.9$_{\pm3.2}$} \\

DAPO
& 100.0 & 70.3 & 90.6
& 91.3 & 86.6 & 82.9 & 88.5$_{\pm1.2}$ 
& 86.5$_{\pm0.9}$ & 75.9$_{\pm2.9}$ \\

\cellcolor{blue!10}$\quad+$AEM
& \cellcolor{blue!10}97.3
& \cellcolor{blue!10}90.3
& \cellcolor{blue!10}98.8
& \cellcolor{blue!10}98.4
& \cellcolor{blue!10}90.9
& \cellcolor{blue!10}89.5
& \cellcolor{blue!10}\textbf{94.5$_{\pm1.4}$} 
& \cellcolor{blue!10}\textbf{88.0$_{\pm1.0}$}
& \cellcolor{blue!10}\textbf{78.5$_{\pm1.0}$} \\

\midrule
\multicolumn{10}{c}{\textit{Qwen2.5-7B-Instruct}} \\
ReAct
& 48.5 & 35.4 & 34.3 & 13.2 & 18.2 & 17.6 & 31.2 
& 46.2 & 19.5 \\

PPO
& 92.3 & 64.0 & 92.5 
& 89.5 & 80.3 & 68.8 & 80.4$_{\pm2.7}$
& 81.4$_{\pm3.1}$ & 68.7$_{\pm5.1}$ \\

GRPO 
& 91.3 & 91.5 & 79.9
& 76.9 & 75.2 & 44.3 & 78.7$_{\pm1.6}$
& 84.1$_{\pm 2.5}$ & 75.9$_{\pm 3.4}$ \\

\cellcolor{blue!10}$\quad+$AEM 
& \cellcolor{blue!10}98.9
& \cellcolor{blue!10}78.6
& \cellcolor{blue!10}89.4
& \cellcolor{blue!10}84.1
& \cellcolor{blue!10}79.5
& \cellcolor{blue!10}65.7
& \cellcolor{blue!10}\textbf{84.4$_{\pm3.1}$}
& \cellcolor{blue!10}\textbf{86.9$_{\pm1.4}$}
& \cellcolor{blue!10}\textbf{80.5$_{\pm2.1}$} \\

GSPO
& 95.1 & 66.9 & 73.9 
& 80.0 & 79.8 & 69.7 & 80.7$_{\pm2.3}$
& 80.4$_{\pm 1.9}$ & 71.6$_{\pm4.6}$  \\

\cellcolor{blue!10}$\quad+$AEM 
& \cellcolor{blue!10}88.9 & \cellcolor{blue!10}56.8 & \cellcolor{blue!10}92.6 
& \cellcolor{blue!10}85.2 & \cellcolor{blue!10}84.8 & \cellcolor{blue!10}78.3 & \textbf{\cellcolor{blue!10}83.4$_{\pm 3.1}$}
&  \cellcolor{blue!10}\textbf{81.9$_{\pm 1.0}$} & \cellcolor{blue!10}\textbf{72.1$_{\pm3.0}$}\\

DAPO
& 100.0 & 96.3 & 100.0
& 94.7 & 90.3 & 94.3 & 96.1$_{\pm2.1}$
& 93.7$_{\pm0.5}$ & 86.7$_{\pm1.4}$ \\

\cellcolor{blue!10}$\quad+$AEM
& \cellcolor{blue!10}99.0
& \cellcolor{blue!10}91.7
& \cellcolor{blue!10}100.0
& \cellcolor{blue!10}96.3
& \cellcolor{blue!10}95.2
& \cellcolor{blue!10}93.2
& \cellcolor{blue!10}\textbf{96.6$_{\pm0.7}$}
& \cellcolor{blue!10}\textbf{94.5$_{\pm1.0}$}
& \cellcolor{blue!10}\textbf{88.9$_{\pm0.9}$} \\
\bottomrule
\end{tabularx}
\label{tab:alfworld_webshop}
\end{table}
\subsection{Overall Performance}\label{sec:performance}

\paragraph{Performance on ALFWorld and WebShop.}

Table~\ref{tab:alfworld_webshop} reports the overall results of applying \name to different baselines on ALFWorld and WebShop.
Overall, AEM consistently improves group-based RL baselines across both benchmarks and model scales, and in several settings achieves performance competitive with strong closed-source models. These results validate adaptive entropy modulation as an effective plug-in mechanism for multi-turn agent training.
By modulating advantages with response-level uncertainty, \name provides denser credit assignment for GRPO and yields consistent gains of 8.8\% (5.7\%) and 5.6\% (4.6\%) on ALFWorld and WebShop, respectively, using 1.5B (7B) models without any extra supervision.
The results show that DAPO provides a stronger group-based optimization backbone than GRPO. Nevertheless, DAPO still benefits from \name{}, achieving additional gains of up to 6.0\%, suggesting that modulation of entropy-aware responses remains complementary even to more advanced optimization backbones: DAPO improves \emph{the way} updates are performed, while \name{} determines \emph{which responses} should receive stronger learning signals during training.
Moreover, \name{} further improves GSPO by up to 5.4\%, suggesting that entropy-aware credit assignment remains complementary even when applied on top of response-level optimization.
The training curves are deferred to Appendix~\ref{appendix:curves}.

\paragraph{Performance on SWE-bench-Verified.}
To further validate the effectiveness of \name in larger-scale models and more challenging tasks, we evaluated it on SWE-bench-Verified and compared it with DeepSWE.
DeepSWE performs RL on Qwen3-32B using the R2E dataset~\citep{jain2024r2e}, and reports a 42.2\% resolved rate on SWE-bench-Verified at the time of release. In our reproduction, DeepSWE achieves an average resolved rate of 42.3\%, serving as a strong baseline for evaluating \name.
\begin{wraptable}{r}{0.34\textwidth}
    \vspace{-0.0em}
    \centering
    \small
    \caption{SWE-bench-Verified results with Qwen3-32B.}
    \label{tab:swebench_main}
    \begin{tabular}{lc}
        \toprule
        Method & Resolved (\%) \\
        \midrule
        DeepSWE & 42.3$_{\pm0.3}$ \\
        \cellcolor{blue!10}DeepSWE+\name & \cellcolor{blue!10}\textbf{43.7$_{\pm0.4}$} \\
        \bottomrule
    \end{tabular}
    \vspace{-0.6em}
\end{wraptable}
As shown in Table~2, DeepSWE+AEM achieves a resolved rate of 43.7\%, outperforming DeepSWE by 1.4\%. SWE-bench-Verified is substantially more challenging than ALFWorld and WebShop, with large solution spaces, and open-ended software environments. Improvements on this benchmark suggest that \name{} remains effective beyond controlled agent benchmarks, extending to realistic multi-turn settings that resemble production workloads.

\subsection{Analysis}\label{sec:analysis}
\noindent

\phantomsection
\label{sec:consistency}
\paragraph{Analysis A: Consistency between $\alpha-1$ and $-(S-\Hresp)$.}

To examine whether $\alpha-1$ matches the sign of $-(S-\Hresp)$,
we conduct a Monte Carlo probing study on the relationship between
$\alpha-1$ and $S(a\mid s)-H_{\mathrm{resp}}(s)$.
We probe $n=64$ states, and for each state we sample $K=64$ responses
to estimate the Monte Carlo (MC) response-level surprisal
\(
\mathcal H_{\mathrm{resp}}^{\text{MC}}(s)
=
\frac{1}{K}\sum_{j=1}^{K} S(a_j\mid s).
\)
We then compare $\alpha-1$ with MC relative surprisal
\begin{equation*}
\Delta S^{\text{MC}}
:=
-\bigl(
S(a\mid s)-\mathcal H_{\mathrm{resp}}^{\text{MC}}(s)
\bigr)
\approx -(S-\Hresp).
\end{equation*}
As illustrated in Figure~\ref{fig:probe}, $\alpha-1$ shows a clear
positive relationship with this quantity, with Pearson correlation
$r=0.63$.
Moreover, the sign of $\alpha-1$ agrees with the sign of
$\Delta S^{\text{MC}}$ in 55 out of 64 states (85.9\%).
These results suggest that $\alpha-1$ is strongly consistent with
$\Delta S^{\text{MC}}$ and provides a practical partition of responses
into the two sides of the relative surprisal.

\begin{figure}[htbp]
\centering
\begin{minipage}[t]{0.45\linewidth}
    \vspace{0pt}
    \centering
    \includegraphics[width=\linewidth]{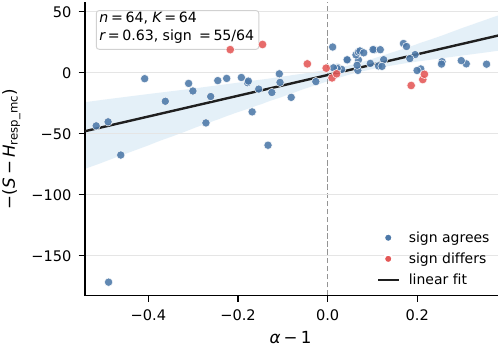}
    \caption{Empirical relationship between $\alpha-1$ and the
    Monte Carlo relative surprisal
    $-(S-H_{\mathrm{resp}}^{\text{MC}})$.}
    \label{fig:probe}
\end{minipage}\hfill
\begin{minipage}[t]{0.45\linewidth}
    \vspace{0pt}
    \centering
    \includegraphics[width=\linewidth]{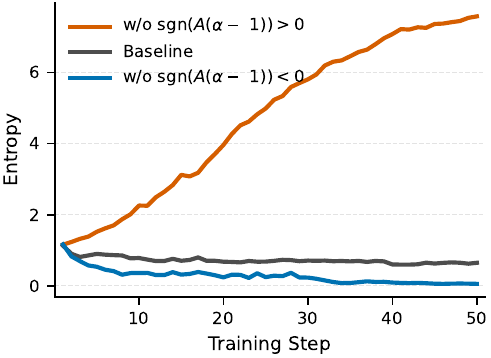}
    \caption{Two masking strategies lead to clearly diverging
    entropy trends.}
    \label{fig:entropy_illustration}
\end{minipage}
\end{figure}

\phantomsection
\label{sec:entropy_illustration}
\paragraph{Analysis B: Validating the trend of entropy.}

To further illustrate how $A(\alpha-1)$ governs the direction of
entropy dynamics, Figure~\ref{fig:entropy_illustration} visualizes
the entropy trend over the first 50 training steps under two
gradient-masking strategies during GRPO training.
Masking the responses such that $(\alpha>1,A>0)$ and
$(\alpha<1,A<0)$, and masking the responses with opposite sign leads to
entropy decrease, i.e.,

\begin{equation*}
\left\{
\begin{aligned}
&\text{Masking }\operatorname{sgn}\tilde D_{\mathrm{RL}}=-\operatorname{sgn}\bigl(A(\alpha-1)\bigr)=-1\\
&\implies \text{entropy increases.}
\end{aligned}
\right.
\qquad
\left\{
\begin{aligned}
&\text{Masking }\operatorname{sgn}\tilde D_{\mathrm{RL}}=-\operatorname{sgn}\bigl(A(\alpha-1)\bigr)=1\\
&\implies \text{entropy decreases.}
\end{aligned}
\right.
\end{equation*}
With the result of \hyperref[sec:consistency]{Analysis A},
this pattern is consistent with Eq.~\eqref{eq:practical},
suggesting that the trend of entropy dynamics is jointly
determined by $A(a)$ and $\alpha-1$.

\paragraph{Analysis C: AEM induces an exploration-exploitation transition.}

Figure~\ref{fig:six-entropy} shows the entropy dynamics across multiple runs. The baseline exhibits an abrupt entropy collapse at the beginning of training and then remains in a relatively flat entropy regime, indicating premature concentration and limited late-stage optimization. In contrast, AEM consistently preserves higher entropy in the early stage and gradually reduces it to a lower range later, suggesting a systematic transition from exploration to exploitation rather than an isolated run-specific effect.

To better understand this transition, Figure~\ref{fig:exploration} overlays entropy with success rate for a representative pair of runs. AEM maintains higher entropy early on, promoting response diversity. As the success rate increases during training, the training batches contain a growing proportion of positive samples relative to negative ones, under which AEM gradually transitions from entropy-increasing to entropy-decreasing dynamics adaptively. This enables the policy to exploit the diversity accumulated during early exploration and achieve a higher final success rate. In contrast, the baseline collapses entropy prematurely but shows limited further improvement, remaining in a locally suboptimal regime.

\phantomsection
\label{sec:exploration}
\noindent
\begin{minipage}[t]{0.48\linewidth}
    \centering
    \includegraphics[width=\linewidth]{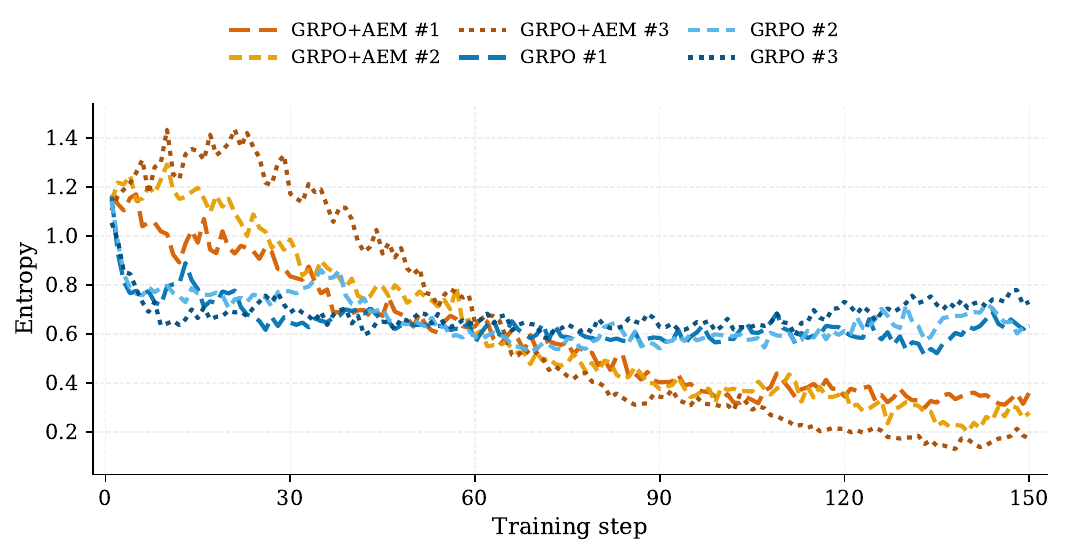}
    \captionof{figure}{Entropy trajectories over training for GRPO and GRPO+AEM.}
    \label{fig:six-entropy}
\end{minipage}\hfill
\begin{minipage}[t]{0.48\linewidth}
    \centering
    \includegraphics[width=\linewidth]{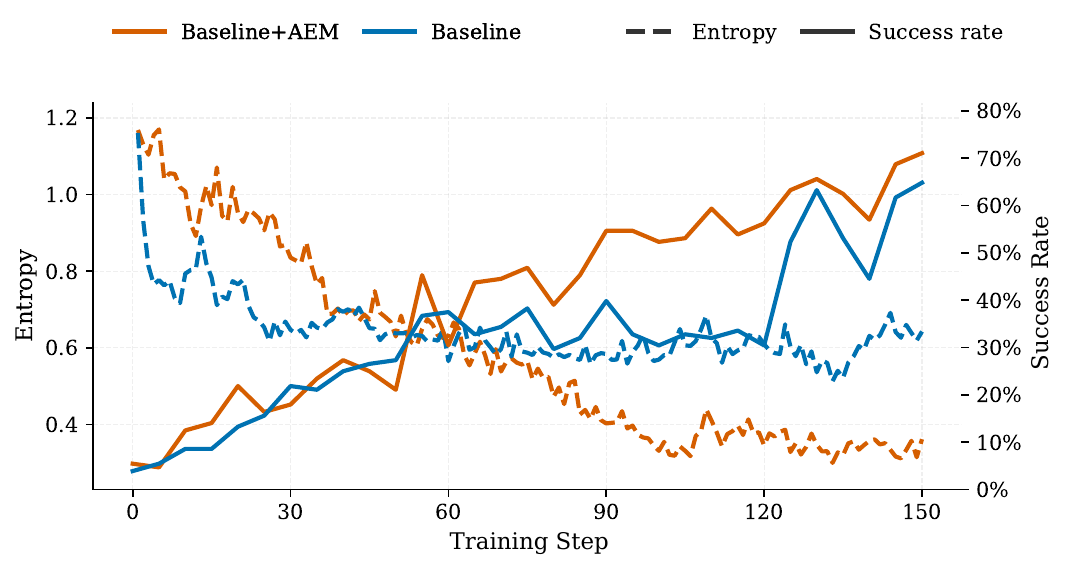}
    \captionof{figure}{Entropy and success-rate dynamics for one pair of runs.}
    \label{fig:exploration}
\end{minipage}

\begin{wrapfigure}{r}{0.48\textwidth}
    \centering
    \vspace{-1.0em}
    \includegraphics[width=0.48\textwidth]{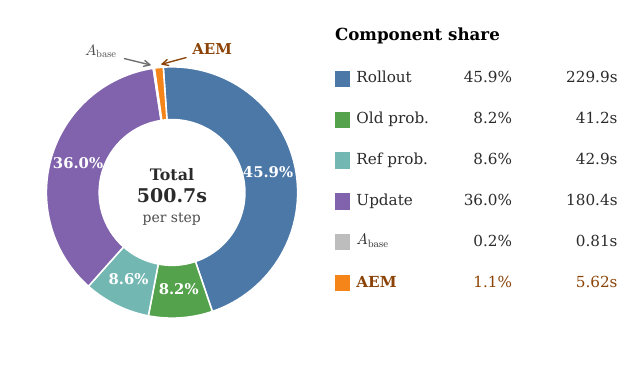}
    \caption{Training time breakdown of GRPO+\name.}
    \label{fig:aem_overhead}
    \vspace{-1.0em}
\end{wrapfigure}
\subsection{Computational Cost}
\label{subsec:comp_cost}
This section analyzes the additional computational overhead introduced by \name. The extra cost is limited to lightweight response-level uncertainty estimation and modulation, including response-level entropy aggregation, group-wise normalization, and advantage rescaling. 
Importantly, \name requires neither extra rollouts nor additional policy or reference model forward passes. The entropy values used by \name are obtained during the same recomputation pass used to compute old-policy log-probabilities, and therefore incur no additional model forward pass. 
Figure~\ref{fig:aem_overhead} reports a detailed per-iteration latency breakdown for training Qwen2.5-1.5B on ALFWorld with GRPO+\name. The overall training time is dominated by rollout generation, model updates, and log-probability computation, which account for approximately 45.9\%, 36.0\%, and 16.8\% of the policy latency, respectively. In contrast, \name-specific computations account for only 1.1\%, indicating that \name introduces negligible overhead in practice.

\section{Conclusions}

This paper presents \name, a supervision-free credit assignment method for multi-turn agentic RL that uses response-level entropy as an intrinsic signal.
Our analysis shows that entropy dynamics are governed by the interaction between advantage and relative response surprisal, which motivates an adaptive entropy modulation rule to regulate entropy dynamics and enables a natural transition from exploration to exploitation. As a lightweight plug-in to existing advantage estimators, it improves credit assignment without auxiliary models, dense supervision, or restrictive structural assumptions. Across ALFWorld, WebShop, and SWE-bench-Verified, \name consistently improves strong baselines, mitigates premature entropy collapse, and yields stronger final performance. These results highlight response-level entropy not only as a useful lens for understanding multi-turn agent training, but also as a practical mechanism for adaptive exploration-exploitation control.

\section*{Disclosures and Acknowledgments}
\paragraph{Disclosures.} This work is supported by Baidu Inc and the Tsinghua University Education Foundation.
\paragraph{Acknowledgments.} We sincerely thank Peng Li from the Institute for AI Industry Research (AIR), Tsinghua University, for his valuable suggestions and insightful discussions, which helped greatly improve this work. We thank Mingzhe Lu from the University of Chinese Academy of Sciences for his valuable advice on refining the paper presentation. 
We also sincerely thank Annan Li and Xiaomin Yuan from the Baidu FAMOU Institute for their generous encouragement in pursuing this project. Lastly, Songlin Zhou thanks Tiran Cai for her consistent support, Haotian Zhao thanks Xiaofeng Wang for all the things.




{\small
\bibliographystyle{plainnat}
\bibliography{refs}
}

\appendix
\newpage
\section{Pseudo-code Algorithm}\label{appendix:AEM}
The pseudo-code algorithm of \name is illustrated in Algorithm~\ref{alg:AEM}.
\begin{algorithm}
\caption{\name}
\label{alg:AEM}
\begin{algorithmic}[1]
\REQUIRE Sampled rollouts $\{\tau_i\}_{i=1}^B$ in the current batch $\mathcal B$, $t$-th response $(i,t)$ in $i$-th trajectory, entropy $\{\mathcal{H}_{\ell,t}^i\}$ of $\ell$-th token in response $(i,t)$, base advantages $\{A^{\mathrm{base}}_{i,t}\}$, temperature $\lambda$, stability constant $\varepsilon$.
\ENSURE Modulated AEM advantages $\{A^{\mathrm{AEM}}_{i,t}\}$
\STATE Parse rollouts $\{\tau_i\}$ into environment-reactive agentic responses $\mathcal{S}_i = \{S_{i,1}, \dots, S_{i,K_i}\}$
\FORALL{rollout $i$ and response $t \in \{1,\dots,K_i\}$}
    \STATE Compute response-level uncertainty proxy: $\bar{\mathcal{H}}_{i,t} \leftarrow \frac{1}{|S_{i,t}|} \sum_{\ell \in S_{i,t}} \mathcal{H}_{\ell,t}^i$
\ENDFOR
\FORALL {groups $\mathcal G \subset \mathcal B$}
    \STATE Find group extrema: $\bar{\mathcal{H}}^{\min}_\mathcal G \leftarrow \min_{(j,n) \in \mathcal G} \bar{\mathcal{H}}_{j,n}$ \textbf{and} $\bar{\mathcal{H}}^{\max}_\mathcal G \leftarrow \max_{(j,n) \in \mathcal G} \bar{\mathcal{H}}_{j,n}$
    \IF{$\bar{\mathcal{H}}^{\max}_\mathcal G - \bar{\mathcal{H}}^{\min}_\mathcal G < 0.1$}
        \FORALL{responses $(i,t) \in \mathcal G$}
            \STATE Set the coefficient: $\alpha_{i,t} \leftarrow 1$
        \ENDFOR
    \ELSE
        \FORALL{responses $(i,t) \in \mathcal G$}
            \STATE Min-max normalization: $\tilde{\mathcal{H}}_{i,t} \leftarrow (\bar{\mathcal{H}}_{i,t} - \bar{\mathcal{H}}^{\min}_\mathcal G) \,/\, (\bar{\mathcal{H}}^{\max}_\mathcal G - \bar{\mathcal{H}}^{\min}_\mathcal G + \varepsilon)$
            \STATE Compute raw modulation coefficient: $\alpha_{i,t} \leftarrow \exp(-\lambda \tilde{\mathcal{H}}_{i,t})$
        \ENDFOR
        \STATE Compute group-average coefficient: $\bar{\alpha}_\mathcal G \leftarrow \frac{1}{|\mathcal G|} \sum_{(j,n) \in \mathcal G} \alpha_{j,n}$
        \FORALL{responses $(i,t) \in \mathcal G$}
            \STATE $\alpha_{i,t} \leftarrow \alpha_{i,t} \,/\, (\bar{\alpha}_\mathcal G + \varepsilon)$
        \ENDFOR
    \ENDIF
\ENDFOR
\FORALL{rollouts $i$, responses $t$}
    \STATE Apply response-level uniform modulation: $A^{\mathrm{AEM}}_{i,t} \leftarrow \alpha_{i,t} A^{\mathrm{base}}_{i,t}$
\ENDFOR
\STATE \textbf{return} $\{A^{\mathrm{AEM}}_{i,t}\}$
\end{algorithmic}
\end{algorithm}

\section{Limitations}\label{appendix:limitations}
Theoretically, the "frozen rollout distribution" assumption in Theorem~\ref{thm:fisher} is designed to isolate changes in entropy that occur while the distribution remains constant, effectively filtering out entropy variations stemming from shifts in the distribution itself. While this approach is \textit{exact} for the widely used optimized fixed-batch surrogate—and is refreshed at each subsequent rollout iteration—it remains a simplified analytical framework when applied to other policy optimization algorithms. The natural-gradient analysis aims to isolate a parameterization-invariant mechanism at the level of the policy distribution; $B_{ker}$ in Appendix~\ref{appendix:parametrized_version} quantifies the gap between this natural-gradient perspective and the actual first-order parametric optimization analysis.

Pratically, \(H_{\mathrm{resp}}(s)\) is not directly computable for open-ended LLM policies, as it would require summing over the entire response space. We therefore approximate the relative response surprisal with a group-based, length-normalized entropy proxy. While our experiments provide statistical evidence that this proxy is aligned with the desired entropy dynamics and improves training, it is still a heuristic surrogate rather than an exact estimator. Consequently, AEM does not guarantee optimal entropy modulation, and its behavior may depend on the quality and diversity of the sampled rollout group. Designing more accurate estimators of response-level relative surprisal is a promising direction for future work.

\section{Broader Impact}\label{appendix:broader}

This work studies credit assignment in multi-turn agentic reinforcement learning and proposes AEM, a supervision-free, lightweight, and plug-in method for entropy-aware response-level credit modulation. By improving credit assignment under sparse outcome-only rewards with different backbones, AEM may help make LLM agents more effective in long-horizon interaction settings such as web navigation, embodied assistance, and software engineering. More broadly, methods that improve training efficiency without requiring additional dense supervision or auxiliary reward models may reduce engineering complexity and lower the cost of developing capable interactive agents. However, as with any advancement in agent capabilities, AEM may also increase the capability of LLM agents in high-impact domains, which could amplify risks if such agents are deployed without sufficient oversight. In particular, more efficient training of long-horizon interactive agents could facilitate misuse in settings such as autonomous web interaction, large-scale automation, or software manipulation.

Overall, we believe the impact of this work contributes a valuable tool to improve the reliability and sample efficiency of agentic RL research.

\section{Direct Comparison with EMPG}
\label{sec:empg_comparison}

To make empirical analysis of AEM against the closely related entropy-aware method EMPG~\citep{wang2025harnessing}, we directly compare them on ALFWorld and WebShop,
using Qwen2.5-1.5B as the backbone and GRPO as the common base
advantage estimator.
All shared training and evaluation settings are identical across
methods; each method retains only its method-specific modulation
parameters.
As shown in Table~\ref{tab:empg_comparison}, AEM outperforms EMPG
by $7.0\%$ in aggregate ALFWorld performance,
$4.8\%$ in WebShop score, and $3.9\%$ in
WebShop success rate.

\begin{table}[htbp]
    \centering
    \caption{Direct comparison of AEM and EMPG.}
    \label{tab:empg_comparison}
    \small
    \setlength{\tabcolsep}{4pt}
    \renewcommand{\arraystretch}{1.15}
    \resizebox{\linewidth}{!}{%
        \begin{tabular}{lccccccc|cc}
            \toprule
            & \multicolumn{7}{c}{ALFWorld}
            & \multicolumn{2}{c}{WebShop} \\
            \cmidrule(lr){2-8}
            \cmidrule(lr){9-10}
            Method
            & Pick & Look & Clean & Heat & Cool & Pick2
            & All & Score & Succ.\ (\%) \\
            \midrule
            GRPO
            & 78.2 & 49.9 & 70.5 & 72.0 & 75.0 & 39.2
            & $68.0_{\pm 0.8}$
            & $83.6_{\pm 0.2}$
            & $65.0_{\pm 0.6}$ \\
            $\quad$+EMPG
            & 80.8 & 51.3 & 63.6 & 80.9 & 61.7 & 56.3
            & $69.8_{\pm 1.6}$
            & $81.6_{\pm 1.6}$
            & $66.7_{\pm 1.2}$ \\
            \cellcolor{blue!10}$\quad$+AEM
            & \cellcolor{blue!10}88.6 & \cellcolor{blue!10}67.6 & \cellcolor{blue!10}76.4 & \cellcolor{blue!10}60.9 & \cellcolor{blue!10}76.7 & \cellcolor{blue!10}69.9
            & \cellcolor{blue!10}$\mathbf{76.8}_{\pm 1.8}$
            & \cellcolor{blue!10}$\mathbf{86.4}_{\pm 2.1}$
            & \cellcolor{blue!10}$\mathbf{70.6}_{\pm 2.4}$ \\
            \bottomrule
        \end{tabular}%
    }
\end{table}

\newpage
\section{Experimental Training Curves}\label{appendix:curves}

\begin{figure}[h!]
    \centering
    \centering
    \begin{subfigure}[t]{0.32\linewidth}
        \centering
        \includegraphics[width=\linewidth]{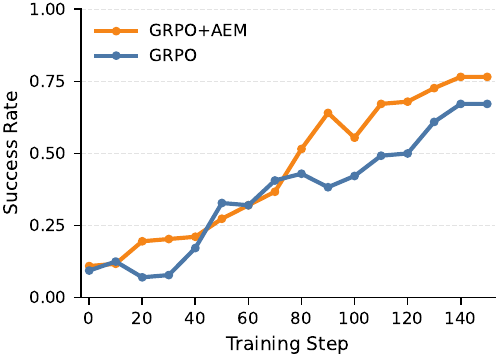}
    \end{subfigure}
    \hfill
    \begin{subfigure}[t]{0.32\linewidth}
        \centering
        \includegraphics[width=\linewidth]{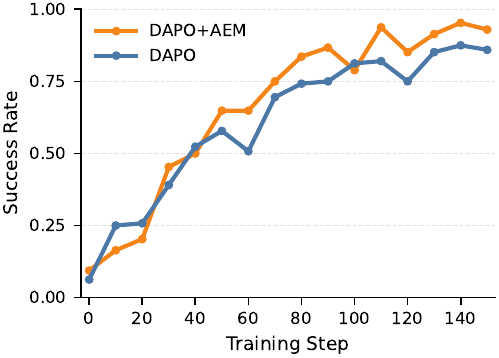}
    \end{subfigure}
    \hfill
    \begin{subfigure}[t]{0.32\linewidth}
        \centering
        \includegraphics[width=\linewidth]{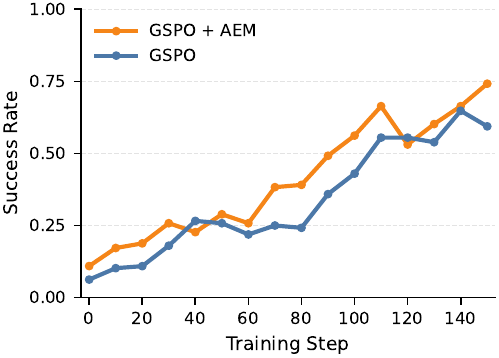}
    \end{subfigure}
    \caption{Training Curves of Qwen2.5-1.5B Model on ALFWorld.}
    \label{fig:append_1.5_alfworld}
\end{figure}

\begin{figure}[h!]
    \centering
    \centering
    \begin{subfigure}[t]{0.32\linewidth}
        \centering
        \includegraphics[width=\linewidth]{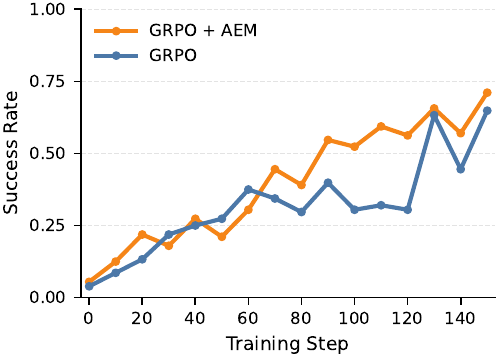}
    \end{subfigure}
    \hfill
    \begin{subfigure}[t]{0.32\linewidth}
        \centering
        \includegraphics[width=\linewidth]{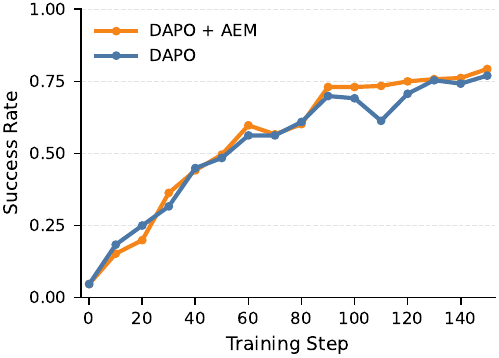}
    \end{subfigure}
    \hfill
    \begin{subfigure}[t]{0.32\linewidth}
        \centering
        \includegraphics[width=\linewidth]{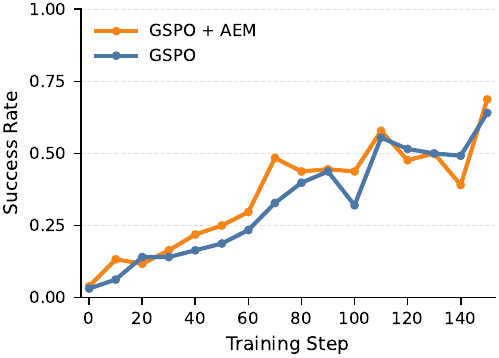}
    \end{subfigure}
    \caption{Training Curves of Qwen2.5-1.5B Model on WebShop.}
    \label{fig:append_1.5_webshop}
\end{figure}

\begin{figure}[h!]
    \centering
    \centering
    \begin{subfigure}[t]{0.32\linewidth}
        \centering
        \includegraphics[width=\linewidth]{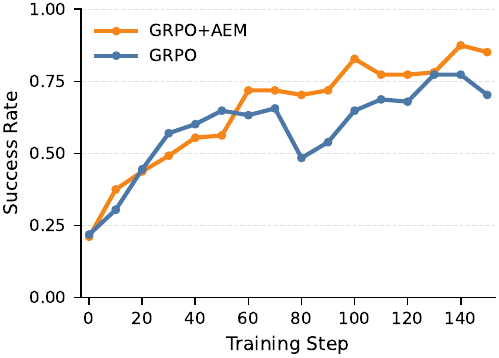}
    \end{subfigure}
    \hfill
    \begin{subfigure}[t]{0.32\linewidth}
        \centering
        \includegraphics[width=\linewidth]{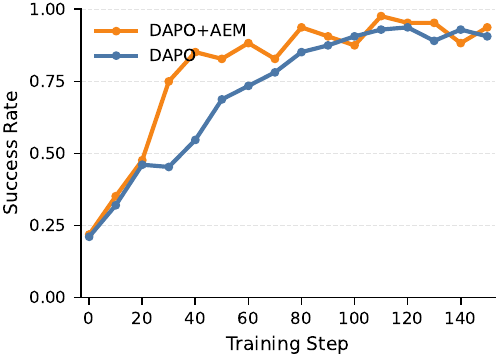}
    \end{subfigure}
    \hfill
    \begin{subfigure}[t]{0.32\linewidth}
        \centering
        \includegraphics[width=\linewidth]{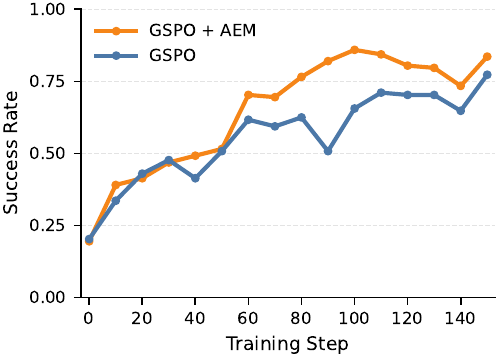}
    \end{subfigure}
    \caption{Training Curves of Qwen2.5-7B Model on ALFWorld.}
    \label{fig:append_7_alfworld}
\end{figure}

\begin{figure}[h!]
    \centering
    \centering
    \begin{subfigure}[t]{0.32\linewidth}
        \centering
        \includegraphics[width=\linewidth]{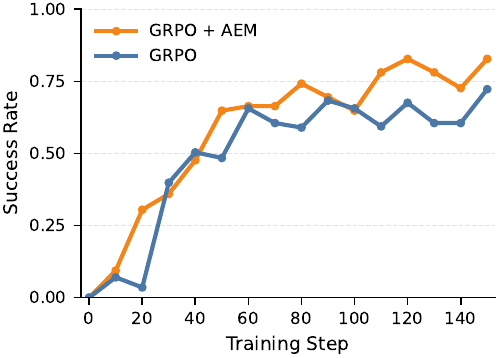}
    \end{subfigure}
    \hfill
    \begin{subfigure}[t]{0.32\linewidth}
        \centering
        \includegraphics[width=\linewidth]{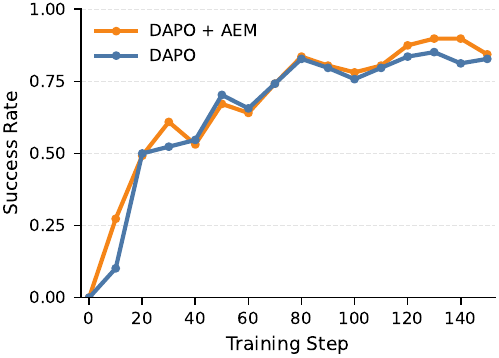}
    \end{subfigure}
    \hfill
    \begin{subfigure}[t]{0.32\linewidth}
        \centering
        \includegraphics[width=\linewidth]{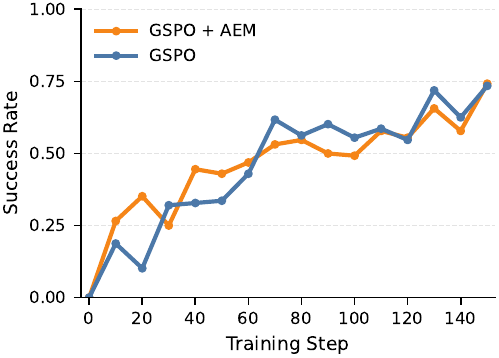}
    \end{subfigure}
    \caption{Training Curves of Qwen2.5-7B Model on WebShop.}
    \label{fig:append_7_webshop}
\end{figure}

\begin{figure}[h!]
    \centering
    \includegraphics[width=0.6\linewidth]{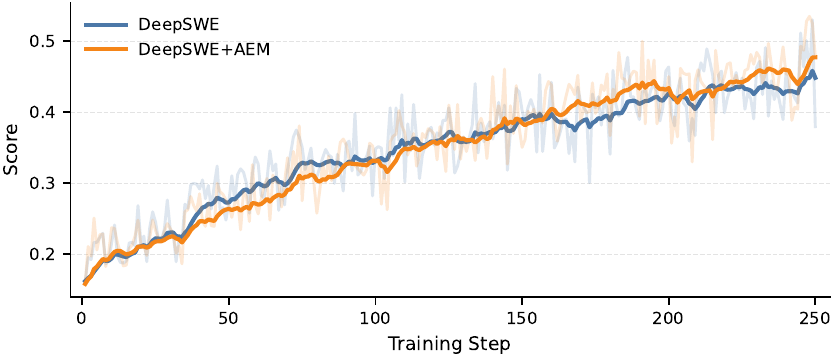}
    \caption{Training reward curves of DeepSWE with and without \name on the R2E dataset.}
    \label{fig:deepswe-curve}
\end{figure}

\newpage
\section{Ablation Study}\label{appendix:ablation}
To better understand the contribution of each component in AEM, we conduct controlled ablation studies on the WebShop benchmark using Qwen2.5-1.5B model. All variants use the same training configuration, rollout budget, and evaluation protocol as the main experiments; the only difference lies in how the modulation coefficient $\alpha$ is constructed or applied.

We compare the following variants:
\begin{enumerate}
\item\textbf{GRPO:} This is the base RL method without entropy-aware advantage modulation.

\item\textbf{+AEM:} This is exactly GRPO+AEM.

\item\textbf{+AEM$_{\text{shuffle}}$:} This variant first computes $\alpha$ in the same way as AEM, but then randomly permutes the coefficients within each group before applying them to response advantages. This preserves the marginal distribution and scale of $\alpha$, but destroys the alignment between each response and its own uncertainty estimate.

\item\textbf{+AEM$_{\text{reverse}}$:} This variant reverses the entropy-dependent modulation rule. Specifically, it changes the temperature $\lambda$ from $1$ to $-1$:
\begin{equation}
\alpha_{i,t}
=
\frac{
\exp(\tilde{\mathcal H}_{i,t})
}{
\frac{1}{|\mathcal G|}\sum_{(j,n)\in\mathcal G} \exp(\tilde{\mathcal H}_{j,n})+\varepsilon
},\quad  \text{for } (i,t) \in \mathcal G.
\end{equation}

\item\textbf{+AEM$_{\text{traj-norm}}$:} This variant sets trajectory normalization instead of group normalization. 
\item\textbf{+AEM$_{\text{batch-norm}}$:} This variant sets batch normalization instead of group normalization. 

\end{enumerate}

The ablation results are reported in Table~\ref{tab:webshop_qwen15b_horizontal}. Overall, full AEM achieves the best performance on both Score and Success Rate, confirming the effectiveness of entropy-aware credit modulation. In contrast, improper modulation strategies reduce, and in some cases even negate, these gains.

Specifically, +AEM$_{\text{shuffle}}$ is still comparable with the GRPO baseline, but remains clearly worse than +AEM. This suggests that the improvement does not come merely from introducing an additional fine-grained rescaling of response-level advantages. Instead, the key factor is whether the entropy signal is assigned to the corresponding response. Once this alignment is destroyed, even if the marginal distribution and scale of the modulation coefficients are preserved, the benefit is substantially weakened.
+AEM$_{\text{reverse}}$ performs substantially worse than GRPO, indicating that an incorrect entropy-to-credit mapping is actively harmful. Intuitively, this reversed mapping tends to exacerbate entropy collapse in the early stage of training, while suppressing beneficial exploitation later on. This result shows that what matters is not rescaling advantages, but applying the entropy-aware credit assignment in the correct direction.

The normalization results show that among the three choices, the group normalization is the most suitable for AEM. Compared with trajectory-level normalization +AEM$_{\text{traj-norm}}$, it benefits from stronger statistics by aggregating multiple responses. Compared with batch-level normalization +AEM$_{\text{batch-norm}}$, it avoids the potential entropy bias caused by mixing tasks, since all normalized responses come from the same prompt. This makes the entropy values more comparable and leads to more effective entropy-aware credit assignment.

\begin{table*}[t]
\centering
\small
\setlength{\tabcolsep}{6pt}
\begin{tabular}{lcccccc}
\toprule
Metric & GRPO & +AEM & +AEM$_{\text{reverse}}$ & +AEM$_{\text{shuffle}}$ & +AEM$_{\text{traj-norm}}$ & +AEM$_{\text{batch-norm}}$ \\
\midrule
Score 
& 83.6$_{\pm 0.2}$ 
& \textbf{86.4$_{\pm 2.1}$} 
& 77.2$_{\pm 3.3}$ 
& 85.6$_{\pm 1.1}$ 
& 83.8$_{\pm 3.1}$ 
& 83.1$_{\pm 4.8}$ \\
Succ. Rate 
& 65.0$_{\pm 0.6}$ 
& \textbf{70.6$_{\pm 2.4}$} 
& 64.5$_{\pm 1.7}$ 
& 64.8$_{\pm 2.4}$
& 68.7$_{\pm 1.5}$ 
& 66.1$_{\pm 2.4}$ \\
\bottomrule
\end{tabular}
\caption{Performance of ablation study on WebShop. Each entry reports the mean and sample standard deviation over 3 runs.}
\label{tab:webshop_qwen15b_horizontal}
\end{table*}

\newpage
\section{Theoretical Details and Proofs}\label{sec:proofs}
In this section, we rigorously provide mathematical details, prove the theorems and properties related to algorithms listed in the main text.

\subsection{Proof of Theorem~\ref{thm:responselevelentropy}}
\label{appendix:response}
\begin{proof}
For brevity, define
\begin{equation}
    Z(a_t,s_t)
    :=
    \sum_{\ell\ge 1}
    \mathcal H_\ell(a_t,s_t)\mathbf 1\{\ell\le |a_t|\}.
\end{equation}

\textbf{Step 1: We show that the response-level entropy is the conditional expectation of the pathwise token-entropy sum.}

By definition, the response-level entropy is
\begin{align}
\Hresp(s_t)
&=
-\sum_{a_t} \pi_\theta(a_t\mid s_t)\log \pi_\theta(a_t\mid s_t).
\end{align}
Since the policy is autoregressive, for any response
\(a=(y_1,\dots,y_{|a|})\),
\begin{align}
\log \pi_\theta(a_t\mid s_t)
&=
\sum_{\ell=1}^{|a_t|}
\log p_\theta(y_\ell\mid s_t,y_{<\ell}).
\end{align}
Therefore,
\begin{align}
\Hresp(s_t)
&=
-\sum_{a_t} \pi_\theta(a_t\mid s_t)\sum_{\ell=1}^{|a_t|}
\log p_\theta(y_\ell\mid s_t,y_{<\ell}) \notag\\
&=
-\sum_{a_t} \sum_{\ell\ge 1}
\pi_\theta(a_t\mid s_t)\,
\log p_\theta(Y_\ell\mid s_t,Y_{<\ell})\,
\mathbf 1\{\ell\le |a_t|\}
\notag\\
&=
\sum_{\ell\ge 1}
\mathbb E_{a_t\sim\pi_\theta(\cdot\mid s_t)}
\left[
-\log p_\theta(Y_\ell\mid s_t,Y_{<\ell})\,
\mathbf 1\{\ell\le |a_t|\}
\;\middle|\; s_t
\right].
\label{eq:proof_resp_step1}
\end{align}
Now apply the tower property. Since
\(\mathbf 1\{\ell\le |a_t|\}\) is measurable with respect to the prefix
\((s_t,Y_{<\ell})\),
\begin{align}
\Hresp(s_t)
&=
\sum_{\ell \ge 1}\mathbb E_{a_t \sim\pi_\theta(\cdot\mid s_t)}
\left[
-\log p_\theta(Y_\ell\mid s_t,Y_{<\ell})\,
\mathbf 1\{\ell\le |a_t|\}
\;\middle|\; s_t
\right] \notag\\
&=
\sum_{\ell \ge 1}\mathbb E_{a_t \sim\pi_\theta(\cdot\mid s_t)}
\left[
\mathbf 1\{\ell\le |a_t|\}\,
\mathbb E
\left[
-\log p_\theta(Y_\ell\mid s_t,Y_{<\ell})
\;\middle|\; s_t,Y_{<\ell}
\right]
\;\middle|\; s_t
\right]
\notag\\
&=
\sum_{\ell \ge 1}\mathbb E_{a_t\sim\pi_\theta(\cdot\mid s_t)}
\left[
\mathcal H_\ell(a_t,s_t)\,\mathbf 1\{\ell\le |a_t|\}
\;\middle|\; s_t
\right]
\notag\\
&=
\mathbb E_{a_t\sim\pi_\theta(\cdot\mid s_t)}
\left[
\sum_{\ell\ge 1}
\mathcal H_\ell(a_t,s_t)\,\mathbf 1\{\ell\le |a_t|\}
\;\middle|\; s_t
\right]
\notag\\
&=
\mathbb E_{a_t\sim\pi_\theta(\cdot\mid s_t)}
\left[
Z(a_t,s_t)\mid s_t
\right].
\label{eq:proof_resp_step1_final}
\end{align}

\textbf{Step 2: We show that the policy entropy is the expected sum of response-level entropies over on-policy visited states.}

Assume on-policy rollouts:
\[
s_0\sim\mathcal D,
\qquad
\tau\sim P_\theta(\cdot\mid s_0),
\]
so that, at each visited state \(s_t\),
\[
a_t\sim\pi_\theta(\cdot\mid s_t).
\]
By the definition of \(\Hpolicy\), using the tower property and Step 1:
\begin{align}
\Hpolicy
&=
\mathbb E_{s_0\sim\mathcal D,\tau\sim P_\theta(\cdot\mid s_0)}
\left[
\sum_{t=0}^{T-1}
Z(a_t,s_t)
\right].
\label{eq:proof_policy_1}\notag\\
&=
\mathbb E_{s_0\sim\mathcal D,\tau\sim P_\theta(\cdot\mid s_0)}
\left[
\sum_{t=0}^{T-1}
\mathbb E
\left[
Z(a_t,s_t)
\,\middle|\,
s_t
\right]
\right].
\notag\\
&=
\mathbb E_{s_0\sim\mathcal D,\tau\sim P_\theta(\cdot\mid s_0)}
\left[
\sum_{t=0}^{T-1}
\Hresp(s_t)
\right].
\end{align}
This proves that the policy entropy under on-policy rollouts is exactly the expected aggregation of response-level entropies over visited states.

Combining Step 1 and Step 2 completes the proof.
\end{proof}

\subsection{Policy Simplex}\label{appendix:simplex}
For a fixed state $s$, with finite action space $|\mathcal A_s| = m$, the policy $\pi = (\pi_\theta(a|s))_{a \in \mathcal A_s}$ is on the simplex
\begin{equation}
    \Delta^\circ(\mathcal A_s)
:=
\left\{
\pi\in\mathbb R^{m}:\sum_{a\in\mathcal A_s} \pi(a)=1, \quad \forall a \in \mathcal A_s,\pi(a)>0
\right\}
\end{equation}

equipped with Fisher-Rao metric to become a Riemannian manifold: for any $u, v \in T_\pi\Delta^\circ(\mathcal A_s)$
\begin{equation}
    g_\pi(u,v) := \sum_{a = 1}^m \frac{u_av_a}{\pi_a}.
\end{equation}
with tangent space
\begin{equation}
   T_\pi\Delta^\circ(\mathcal A_s)
=
\{x\in\mathbb R^{m}:\mathbf 1^\top x=0\}. 
\end{equation}
Fisher-Rao metric is the infinitesimal quadratic form induced by the KL divergence. For any tangent
perturbation \(\delta\in T_\pi\Delta^\circ(\mathcal A_s)\), i.e. \(\sum_a\delta_a=0\), we have
\[
D_{\mathrm{KL}}(\pi+\delta\,\|\,\pi)
=
\frac{1}{2}
\sum_{a\in\mathcal A_s}
\frac{\delta_a^2}{\pi_a}
+
o(\|\delta\|^2)
=
\frac{1}{2}g_\pi(\delta,\delta)
+
o(\|\delta\|^2).
\]
Thus, the Fisher-Rao metric measures the local size of a policy update in the same units as a local KL trust region.

\subsection{State and Proof of the Generalized Version of Theorem~\ref{thm:fisher}}\label{appendix:fisher}

\begin{theorem}[Regularized Response-level entropy drift. Proved in Appendix~\ref{appendix:fisher}]
\label{thm:fisher_real_world}
Let $\operatorname{grad}^F$ denote the natural gradient on the policy simplex, and consider the regularized local objective
\[
\ell_a(\pi)
=
A(a,s)\log \pi
+
\beta\,\psi(\Hresp(\pi))
-
\gamma\,D_{KL}(\pi\|\pi_{\mathrm{ref}}).
\]
Then the directional derivative of $\Hpolicy$ along the update direction $\operatorname{grad}^F \ell_a(\pi)$
\begin{align}
D_{\mathrm{RL}}(a;s)
&:=
\left\langle
\operatorname{grad}^F \Hpolicy(\pi),\,
\operatorname{grad}^F \ell_a(\pi)
\right\rangle_{\mathrm{Fisher\text{-}Rao}} \\
&= \sum_{t=0}^{T-1}\mathbb P_{s_0 \sim \mathcal D, \tau \sim P_\theta(\cdot\mid s_0)}[s_t = s]D_{RL}^{\mathrm{resp}}(a;s)\label{eq:policy_resp}
\end{align}
with $D_{RL}^{\mathrm{resp}}(a;s)$ defined by
\begin{align}
D_{\mathrm{RL}}^{\mathrm{resp}}(a;s)
&=
\underbrace{
A(a,s)\bigl(S(a\mid s)-\Hresp(\pi)\bigr)
}_{\textnormal{(I) reward-driven term}}
+
\underbrace{
(\beta\psi'(\Hresp(\pi))+\gamma)\operatorname{Var}_{a\sim\pi(\cdot|s)}\!\bigl(S(a\mid s)\bigr)
}_{\textnormal{(II) entropy-expanding term}}
\notag\\
&-
\underbrace{
\gamma\,\operatorname{Cov}_{a\sim\pi(\cdot|s)}\!\bigl(S(a\mid s),S_{\mathrm{ref}}(a\mid s)\bigr)
}_{\textnormal{(III) reference-alignment term}} .
\label{eq:thm:fisher3}
\end{align}
If we let $\beta = \gamma = 0$, i.e., only reward objective is considered, then we obtain the Theorem~\ref{thm:fisher}.
\end{theorem}

\begin{remark}
The decomposition in~\eqref{eq:thm:fisher3} yields four immediate observations.
\begin{itemize}
\item 
The entropy and KL regularization terms are \textbf{state-level} modulation terms: unlike the reward-driven term (I), they do not depend on the sampled action signal \(A(a,s)\).
\item
Term (I) shows that the advantage and relative surprisal of sampled action can jointly determine entropy dynamics without entropy and KL regularization. 

\item
The entropy regularizer contributes a positive force through
\(
\beta\psi'(\Hresp)\operatorname{Var}_{a\sim\pi(\cdot|s)}\!\bigl(S(a\mid s)\bigr),
\)
which is consistent with its intended role.
\item The KL term contributes two parts: a positive variance term
\(
\gamma\,\operatorname{Var}_{a\sim\pi(\cdot|s)}\!\bigl(S(a\mid s)\bigr),
\)
and a covariance term
\(
-\gamma\,\operatorname{Cov}_{a\sim\pi(\cdot|s)}\!\bigl(S(a\mid s),S_{\mathrm{ref}}(a\mid s)\bigr),
\)
whose sign is generally not fixed. Fig~\ref{fig:simplex-entropy-drift} demonstrates the entropy dynamics along updating on three-action simplex.
\end{itemize}
\end{remark}

\begin{proof}
Fix a state \(s\). For brevity, write
\[
\pi_b:=\pi(b\mid s),\qquad
\rho_b:=\pi_{\mathrm{ref}}(b\mid s),\qquad
A_a:=A(a,s),
\]
and assume \(\pi_b>0\) and \(\rho_b>0\) for all \(b\in\mathcal A_s\). Define
\[
S_b:=-\log \pi_b,\qquad
S_b^{\mathrm{ref}}:=-\log \rho_b,\qquad
H:=\Hresp(\pi)=\sum_{b\in\mathcal A_s}\pi_b S_b .
\]
We also note
\[
\operatorname{Var}_{a\sim\pi(\cdot|s)}(S)
:=
\sum_{b\in\mathcal A_s}\pi_b(S_b-H)^2,
\]
and
\[
\operatorname{Cov}_{a\sim\pi(\cdot|s)}(S,S_{\mathrm{ref}})
:=
\sum_{b\in\mathcal A_s}
\pi_b(S_b-H)
\bigl(S_b^{\mathrm{ref}}-\mathbb E_\pi[S_{\mathrm{ref}}]\bigr)
=
\sum_{b\in\mathcal A_s}
\pi_b(S_b-H)S_b^{\mathrm{ref}},
\]
where the last equality follows from
\[
\sum_{b\in\mathcal A_s}\pi_b(S_b-H)=0.
\]
\textbf{Step 1.} We first show Eq.~\eqref{eq:policy_resp}:
By Eq.~\eqref{eq:policy} and definition, with the assumption that gradients are not propagated through the rollout distribution \(P_\theta(\tau\mid s_0)\)., since for $s_t \neq s$, $\Hresp(s_t)$ is a constant on the $\Delta^\circ(\mathcal A_s)$, we then deduce:
\begin{align}
    D_{RL}(a;s) &= \left<\operatorname{grad}^F \Hpolicy(\pi),\operatorname{grad}^F \ell_a(\pi)\right>\notag\\
    & = \left<\operatorname{grad}^F \mathbb E_{s_0 \sim \mathcal D, \tau \sim P_\theta}\left[\sum_{t=0}^{T-1}\Hresp(\pi),\operatorname{grad}^F \ell_a(\pi)\right]\right>\notag\\
    &=\left<\operatorname{grad}^F \mathbb E_{s_0 \sim \mathcal D, \tau \sim P_\theta}\left[\sum_{t=0}^{T-1}\mathrm 1(s_t = s)\Hresp(\pi)\right],\operatorname{grad}^F
    \ell_a(\pi)\right>\notag\\
    &=\sum_{t=0}^{T-1}\mathbb E_{s_0 \sim \mathcal D, \tau \sim P_\theta}\left[\mathrm 1(s_t = s)\right]\left<\operatorname{grad}^F \Hresp(\pi),\operatorname{grad}^F
    \ell_a(\pi)\right>\notag\\
    & = \sum_{t=0}^{T-1} \mathbb P_{s_0 \sim \mathcal D, \tau \sim P_\theta}[s_t = s]D_{RL}^{\mathrm{resp}}(a;s),
\end{align}
which is exactly Eq.~\eqref{eq:policy_resp}.

\textbf{Step 2.} For any smooth function
\(f:\Delta^\circ(\mathcal A_s)\to\mathbb R\), its Fisher-Rao gradient is
\begin{equation}
\label{eq:fr_gradient_general}
\operatorname{grad}^F f(\pi)
=
\pi\odot
\Bigl(
\nabla_\pi f
-
(\pi^\top\nabla_\pi f)\mathbf 1
\Bigr).
\end{equation}
where $\odot$ denotes the Hadamard product, and $\mathbf 1$ is the vector with all components to be one. Indeed, for any
\(\xi\in T_\pi\Delta^\circ(\mathcal A_s)\), since
\(\mathbf 1^\top\xi=0\),
\begin{align}
g_\pi\!\left(
\pi\odot
\Bigl(
\nabla_\pi f
-
(\pi^\top\nabla_\pi f)\mathbf 1
\Bigr),
\xi
\right)
&=
\sum_{b\in\mathcal A_s}
\frac{
\pi_b
\bigl(
\partial_{\pi_b}f-\pi^\top\nabla_\pi f
\bigr)
\xi_b
}{\pi_b}
\notag\\
&=
\sum_{b\in\mathcal A_s}
\partial_{\pi_b}f\,\xi_b
-
(\pi^\top\nabla_\pi f)
\sum_{b\in\mathcal A_s}\xi_b
\notag\\
&=
\nabla_\pi f^\top \xi
=
df_\pi[\xi].
\end{align}
Thus~\eqref{eq:fr_gradient_general} is the Riemannian gradient under the Fisher-Rao metric.

\textbf{Step 3.} We compute the Fisher-Rao gradients of all terms in
\[
\ell_a(\pi)
=
A_a\log \pi_a
+
\beta\,\psi(\Hresp(\pi))
-
\gamma\,D_{\mathrm{KL}}(\pi\|\pi_{\mathrm{ref}}).
\]

First, for the reward-driven term
\[
\ell_a^{A}(\pi):=A_a\log \pi_a,
\]
we have
\begin{align}
\operatorname{grad}^F\ell_a^A(\pi)
&=\pi\odot \left(\nabla_\pi\ell_a^A(\pi) - (\pi^\top \nabla_\pi\ell_a^A(\pi))\mathbf 1\right)\notag\\
&=\pi\odot
\left(
\frac{A_a}{\pi_a}e_a
-
A_a\mathbf 1
\right)
\notag\\
&=
A_a(e_a-\pi).
\label{eq:grad_reward_term}
\end{align}

Next, for the response-level entropy, we have
\[
\partial_{\pi_b}\Hresp(\pi)
=
-(1+\log\pi_b)
=
S_b-1,\quad
\pi^\top\nabla_\pi\Hresp(\pi)
=
\sum_{b\in\mathcal A_s}\pi_b(S_b-1)
=
H-1.
\]
Hence,
\begin{align}
\operatorname{grad}^F\Hresp(\pi)
&=
\pi\odot
\Bigl(
(S_b-1)_{b\in\mathcal A_s}
-
(H-1)\mathbf 1
\Bigr)
\notag\\
&=
\pi\odot
\Bigl(
(S_b)_{b\in\mathcal A_s}
-
H\mathbf 1
\Bigr).
\label{eq:grad_entropy}
\end{align}

For the entropy regularizer
\[
\ell^{E}(\pi):=\beta\,\psi(\Hresp(\pi)),
\]
the chain rule gives
\[
\nabla_\pi \ell^{E}(\pi)
=
\beta\,\psi'(H)\nabla_\pi\Hresp(\pi).
\]
Since the Fisher-Rao projection is linear, we have
\begin{align}
\operatorname{grad}^F\ell^{E}(\pi)
&=
\beta\,\psi'(H)\operatorname{grad}^F\Hresp(\pi)
\notag\\
&=
\beta\,\psi'(H)\,
\pi\odot
\Bigl(
(S_b)_{b\in\mathcal A_s}
-
H\mathbf 1
\Bigr).
\label{eq:grad_entropy_reg}
\end{align}

Finally, consider the KL divergence term
\[
K(\pi)
:=
D_{\mathrm{KL}}(\pi\|\pi_{\mathrm{ref}})
=
\sum_{b\in\mathcal A_s}
\pi_b\log\frac{\pi_b}{\rho_b}.
\]
we have
\begin{align}
\operatorname{grad}^FK(\pi)
&= \pi \odot 
\left(\nabla_\pi K(\pi) - (\pi^\top\nabla_\pi K(\pi))\mathbf 1\right)\notag\\
&=\pi\odot\left(
\left(\log\frac{\pi_b}{\rho_b}+1\right)_{b\in\mathcal A_s}
-
(K(\pi)+1)\mathbf 1
\right)
\notag\\
&=
\pi\odot
\left(
\left(\log\frac{\pi_b}{\rho_b}\right)_{b\in\mathcal A_s}
-
K(\pi)\mathbf 1
\right).
\label{eq:grad_kl}
\end{align}
Combining~\eqref{eq:grad_reward_term}, \eqref{eq:grad_entropy_reg}, and~\eqref{eq:grad_kl}, we obtain
\begin{align}
\operatorname{grad}^F\ell_a(\pi)
&=
A_a(e_a-\pi)
+
\beta\,\psi'(H)\,
\pi\odot
\Bigl(
(S_b)_{b\in\mathcal A_s}
-
H\mathbf 1
\Bigr)
\notag\\
&\quad
-\gamma\,
\pi\odot
\left(
\left(\log\frac{\pi_b}{\rho_b}\right)_{b\in\mathcal A_s}
-
K(\pi)\mathbf 1
\right).
\label{eq:grad_full_objective}
\end{align}

\textbf{Step 4.} We compute
\[
D_{\mathrm{RL}}^{\mathrm{resp}}(a,s)
=
g_\pi\!\left(
\operatorname{grad}^F\Hresp(\pi),
\operatorname{grad}^F\ell_a(\pi)
\right).
\]
By~\eqref{eq:grad_entropy} and~\eqref{eq:grad_full_objective},
\begin{align}
D_{\mathrm{RL}}^{\mathrm{resp}}(a,s)
&=
g_\pi
\left(
\pi\odot(S-H\mathbf 1),
A_a(e_a-\pi)
\right)
\notag\\
&\quad
+
\beta\,\psi'(H)\,
g_\pi
\left(
\pi\odot(S-H\mathbf 1),
\pi\odot(S-H\mathbf 1)
\right)
\notag\\
&\quad
-
\gamma\,
g_\pi
\left(
\pi\odot(S-H\mathbf 1),
\pi\odot
\left(
\log\frac{\pi}{\rho}
-
K(\pi)\mathbf 1
\right)
\right).
\label{eq:drift_decomposition}
\end{align}

For the first term,
\begin{align}
g_\pi
\left(
\pi\odot(S-H\mathbf 1),
A_a(e_a-\pi)
\right)
&=
A_a
\sum_{b\in\mathcal A_s}
\frac{
\pi_b(S_b-H)
\bigl((e_a)_b-\pi_b\bigr)
}{\pi_b}
\notag\\
&=
A_a
\left[
S_a-H
-
\sum_{b\in\mathcal A_s}
\pi_b(S_b-H)
\right]
\notag\\
&=
A_a(S_a-H).
\label{eq:reward_drift_term}
\end{align}

For the second term,
\begin{align}
g_\pi
\left(
\pi\odot(S-H\mathbf 1),
\pi\odot(S-H\mathbf 1)
\right)
&=
\sum_{b\in\mathcal A_s}
\frac{
\pi_b^2(S_b-H)^2
}{\pi_b}
\notag\\
&=
\sum_{b\in\mathcal A_s}
\pi_b(S_b-H)^2
\notag\\
&=
\operatorname{Var}_{a\sim\pi(\cdot|s)}(S).
\label{eq:entropy_reg_drift_term}
\end{align}

For the third term, since
\[
\sum_{b\in\mathcal A_s}\pi_b(S_b-H)=0,
\]
we have
\begin{align}
&g_\pi
\left(
\pi\odot(S-H\mathbf 1),
\pi\odot
\left(
\log\frac{\pi}{\rho}
-
K(\pi)\mathbf 1
\right)
\right)
\notag\\
&\qquad
=
\sum_{b\in\mathcal A_s}
\pi_b(S_b-H)
\left(
\log\frac{\pi_b}{\rho_b}
-
K(\pi)
\right)
\notag\\
&\qquad
=
\sum_{b\in\mathcal A_s}
\pi_b(S_b-H)
\log\frac{\pi_b}{\rho_b}.
\label{eq:kl_inner_pre}
\end{align}
Using
\[
\log\frac{\pi_b}{\rho_b}
=
\log\pi_b-\log\rho_b
=
-S_b+S_b^{\mathrm{ref}},
\]
we get
\begin{align}
\sum_{b\in\mathcal A_s}
\pi_b(S_b-H)
\log\frac{\pi_b}{\rho_b}
&=
\sum_{b\in\mathcal A_s}
\pi_b(S_b-H)
\bigl(-S_b+S_b^{\mathrm{ref}}\bigr)
\notag\\
&=
-
\sum_{b\in\mathcal A_s}
\pi_b(S_b-H)S_b
+
\sum_{b\in\mathcal A_s}
\pi_b(S_b-H)S_b^{\mathrm{ref}}
\notag\\
&=
-\operatorname{Var}_{a\sim\pi(\cdot|s)}(S)
+
\operatorname{Cov}_{a\sim\pi(\cdot|s)}(S,S_{\mathrm{ref}}).
\label{eq:kl_inner_final}
\end{align}

Substituting~\eqref{eq:reward_drift_term}, \eqref{eq:entropy_reg_drift_term}, and~\eqref{eq:kl_inner_final} into~\eqref{eq:drift_decomposition}, we obtain
\begin{align}
D_{\mathrm{RL}}^{\mathrm{resp}}(a,s)
&=
A_a(S_a-H)
+
\beta\,\psi'(H)\operatorname{Var}_{a\sim\pi(\cdot|s)}(S)
+
\gamma\operatorname{Var}_{a\sim\pi(\cdot|s)}(S)
-
\gamma\operatorname{Cov}_{a\sim\pi(\cdot|s)}(S,S_{\mathrm{ref}})
\notag\\
&=
A(a,s)
\bigl(S(a\mid s)-\Hresp(\pi)\bigr)
+
\bigl(\beta\psi'(\Hresp(\pi))+\gamma\bigr)
\operatorname{Var}_{a\sim\pi(\cdot|s)}\!\bigl(S(a\mid s)\bigr)
\notag\\
&\quad
-
\gamma
\operatorname{Cov}_{a\sim\pi(\cdot|s)}\!\bigl(S(a\mid s),S_{\mathrm{ref}}(a\mid s)\bigr).
\end{align}
This proves the theorem.
\end{proof}

\subsection{Doob's decomposition of fixed-length response surprisal}
\label{appendix:doob}

\begin{proposition}[Doob's decomposition of response surprisal]
\label{prop:doob_surprisal_fixed_length}
Fix a state \(s\), and let
\(
a=(Y_1,\ldots,Y_L)\sim \pi_\theta(\cdot\mid s)
\)
be a realized response sampled from the policy, where \(L=|a|\) is its realized length.
Define the realized token surprisal
\[
X_\ell
:=
-\log p_\theta(Y_\ell\mid s,Y_{<\ell}),
\]
then the response surprisal admits the decomposition
\[
S(a\mid s)
=
\sum_{\ell=1}^{L}X_\ell
=
\sum_{\ell=1}^{L}\mathcal H_\ell(a,s)
+
M_L,
\]
where
\(
M_k:=\sum_{\ell=1}^{k}\left(X_\ell-\mathcal H_\ell(a,s)\right)
\)
is a zero-mean martingale with respect to \((\mathcal F_k)_{k=0}^{L}\). 
Consequently,
\begin{equation}
    S(a\mid s)-\Hresp(s)
=
\left(
\sum_{\ell=1}^{L}\mathcal H_\ell(a,s)-\Hresp(s)
\right)
+
M_L.
\end{equation}
\end{proposition}

\begin{proof}
For each \(\ell\), by definition,
\begin{align}
\mathbb E[X_\ell\mid\mathcal F_{\ell-1}]
&=
\mathbb E\!\left[
-\log p_\theta(Y_\ell\mid s,Y_{<\ell})
\,\middle|\,
s,Y_{<\ell}
\right]
\notag\\
&=
-\sum_{y\in\mathcal V}
p_\theta(y\mid s,Y_{<\ell})
\log p_\theta(y\mid s,Y_{<\ell})
\notag\\
&=
\mathcal H_\ell(a,s).
\label{eq:fixed_length_predictable_part}
\end{align}
Thus \(\mathcal H_\ell(a,s)\) is \(\mathcal F_{\ell-1}\)-measurable and hence predictable:

\begin{align}
\mathbb E[X_\ell-\mathcal H_\ell(a,s)\mid\mathcal F_{\ell-1}]
&=
\mathbb E[X_\ell-\mathcal H_\ell(a,s)\mid\mathcal F_{\ell-1}]
\notag\\
&=
\mathbb E[X_\ell\mid\mathcal F_{\ell-1}]
-
\mathcal H_\ell(a,s)
\notag\\
&=
0.
\label{eq:fixed_length_mds}
\end{align}
Therefore,
\(
M_k:=\sum_{\ell=1}^{k}X_\ell-\mathcal H_\ell(a,s)
\)
is a martingale.

With the definition of \(X_\ell-\mathcal H_\ell(a,s)\), we obtain
\begin{align}
S(a\mid s)
&=
\sum_{\ell=1}^{L}X_\ell
\notag\\
&=
\sum_{\ell=1}^{L}
\left(
\mathcal H_\ell(a,s)+X_\ell-\mathcal H_\ell(a,s)
\right)
\notag\\
&=
\sum_{\ell=1}^{L}\mathcal H_\ell(a,s)+M_L.
\end{align}
Finally, by the definition of response-level entropy over fixed-length responses,
subtracting \(\Hresp(s)\) from both sides of the Doob's decomposition gives
\[
S(a\mid s)-\Hresp(s)
=
\left(
\sum_{\ell=1}^{L}\mathcal H_\ell(a,s)-\Hresp(s)
\right)
+
M_L.
\]
This completes the proof.
\end{proof}

\subsection{Parametrized Version of Entropy Drift}
\label{appendix:parametrized_version}

In this section, we analyze how the parametrized response entropy varies along the sample-induced update direction in parameter space. The resulting entropy-drift formula is analogous in spirit to the main result in Theorem~\ref{thm:fisher}. However, once the policy is parameterized by $\theta$, the drift additionally involves a kernel-weighted baseline term $B_{\mathrm{ker}}$(a;s).

\begin{theorem}[Parametrized regularized response-level entropy drift]
\label{thm:param_reg_entropy_drift}
Fix a state \(s\). Let
\[
\pi_b:=\pi_\theta(b\mid s),\qquad
\rho_b:=\pi_{\mathrm{ref}}(b\mid s),\qquad
G_b:=\nabla_\theta\log\pi_\theta(b\mid s),
\]
and
\[
S_b:=-\log \pi_b,\qquad
S_b^{\mathrm{ref}}:=-\log \rho_b,\qquad
H:=\Hresp(s)=\sum_{b\in\mathcal A_s}\pi_b S_b.
\]
Define the policy-gradient kernel
\(
K(b,c;s):=\langle G_b,G_c\rangle.
\)
Then the Euclidean parameter-space entropy drift satisfies 
\begin{align}
D_{\mathrm{RL}}^\theta(a;s)
&=
-A(a,s)
\left[
\pi_\theta(a\mid s)(H-S_a)K(a,a;s)
+
B_{\mathrm{ker}}(a;s)
\right]
\notag\\
&\quad
+
\bigl(\beta\psi'(H)+\gamma\bigr)
\mathcal V_\theta(S;s)
-
\gamma\,\mathcal C_\theta(S,S_{\mathrm{ref}};s),
\end{align}
where \(B_{\mathrm{ker}}(a;s)\) is the cross-response residual introduced by shared parameterization:
\begin{align}
\mathcal V_\theta(S;s)
&:=
\E_{b,c\sim\pi_\theta(\cdot\mid s)}
\left[
(S_b-H)(S_c-H)K(b,c;s)
\right]
=
\|\nabla_\theta\Hresp(s)\|_2^2,
\\
\mathcal C_\theta(S,S_{\mathrm{ref}};s)
&:=
\E_{b,c\sim\pi_\theta(\cdot\mid s)}
\left[
(S_b-H)
\bigl(S_c^{\mathrm{ref}}-\E_{\pi_\theta}[S_{\mathrm{ref}}]\bigr)
K(b,c;s)
\right]\\
B_{\mathrm{ker}}(a;s)
&:=
\sum_{b\neq a}
\pi_\theta(b\mid s)(H-S_b)K(b,a;s).
\end{align}
\end{theorem}

\subsection*{Proof of Theorem~\ref{thm:param_reg_entropy_drift}}
\label{appendix:param_reg_entropy}
\begin{proof}
We first prove a general formula for an arbitrary smooth regularizer
\[
\ell_a^{\mathcal R}(\theta)
=
A(a,s)\log\pi_\theta(a\mid s)+\mathcal R(\pi_\theta(\cdot\mid s)).
\]

\textbf{Step 1. Gradient of response-level entropy.}
By definition,
\begin{align}
\nabla_\theta\Hresp(s)
&=
-\sum_{b\in\mathcal A_s}
\nabla_\theta\bigl(\pi_b\log\pi_b\bigr)
\notag\\
&=
-\sum_{b\in\mathcal A_s}
(\log\pi_b+1)\nabla_\theta\pi_b
\notag\\
&=
-\sum_{b\in\mathcal A_s}
\pi_b(\log\pi_b+1)G_b.
\end{align}
Using the zero-score identity
\[
\sum_{b\in\mathcal A_s}\pi_bG_b
=
\sum_{b\in\mathcal A_s}\nabla_\theta\pi_b
=
\nabla_\theta\sum_{b\in\mathcal A_s}\pi_b
=
0,
\]
we have
\begin{align}
\nabla_\theta\Hresp(s)
&=
-\sum_{b\in\mathcal A_s}
\pi_b(\log\pi_b+H)G_b
\notag\\
&=
\sum_{b\in\mathcal A_s}
\pi_b(S_b-H)G_b.
\label{eq:param_grad_H}
\end{align}

\textbf{Step 2. General regularizer.}
Let
\[
r_c:=\partial_{\pi_c}\mathcal R(\pi),
\qquad
\bar r:=\sum_{c\in\mathcal A_s}\pi_c r_c.
\]
By the chain rule,
\begin{align}
\nabla_\theta\mathcal R(\pi_\theta(\cdot\mid s))
&=
\sum_{c\in\mathcal A_s}
\partial_{\pi_c}\mathcal R(\pi)\nabla_\theta\pi_c
\notag\\
&=
\sum_{c\in\mathcal A_s}
\pi_c r_c G_c
\notag\\
&=
\sum_{c\in\mathcal A_s}
\pi_c(r_c-\bar r)G_c,
\label{eq:param_grad_general_R}
\end{align}
Therefore,
\begin{equation}
\nabla_\theta\ell_a^{\mathcal R}
=
A(a,s)G_a
+
\sum_{c\in\mathcal A_s}
\pi_c(r_c-\bar r)G_c.
\label{eq:param_grad_general_objective}
\end{equation}

Taking the inner product between
\eqref{eq:param_grad_H} and \eqref{eq:param_grad_general_objective},
we obtain
\begin{align}
D_{\mathrm{RL}}^{\theta,\mathcal R}(a;s)
&:=
\left\langle
\nabla_\theta\Hresp(s),
\nabla_\theta\ell_a^{\mathcal R}
\right\rangle
\notag\\
&=
A(a,s)
\sum_{b\in\mathcal A_s}
\pi_b(S_b-H)\langle G_b,G_a\rangle
\notag\\
&\quad
+
\sum_{b,c\in\mathcal A_s}
\pi_b\pi_c(S_b-H)(r_c-\bar r)
\langle G_b,G_c\rangle
\notag\\
&=
-A(a,s)
\E_{b\sim\pi_\theta(\cdot\mid s)}
\left[
(H-S_b)K(b,a;s)
\right]
\notag\\
&\quad
+
\E_{b,c\sim\pi_\theta(\cdot\mid s)}
\left[
(S_b-H)(r_c-\bar r)K(b,c;s)
\right].
\label{eq:param_general_R_drift}
\end{align}
This is the general regularized parameter-space entropy-drift identity.

\textbf{Step 3. Apply the general identity to entropy and KL regularization.}
Now take
\[
\mathcal R(\pi)
=
\beta\psi(\Hresp(\pi))
-
\gamma D_{\mathrm{KL}}(\pi\|\pi_{\mathrm{ref}}),
\]
where
\[
D_{\mathrm{KL}}(\pi\|\pi_{\mathrm{ref}})
=
\sum_{c\in\mathcal A_s}
\pi_c\log\frac{\pi_c}{\rho_c},
\qquad
\rho_c:=\pi_{\mathrm{ref}}(c\mid s).
\]
For the entropy term,
\[
\partial_{\pi_c}\Hresp(\pi)
=
-(1+\log\pi_c)
=
S_c-1.
\]
For the KL term,
\[
\partial_{\pi_c}D_{\mathrm{KL}}(\pi\|\pi_{\mathrm{ref}})
=
\log\frac{\pi_c}{\rho_c}+1.
\]
Therefore,
\[
r_c
=\partial_{\pi_c}\mathcal R(\pi) = 
\beta\psi'(H)(S_c-1)
+
\gamma S_c
-
\gamma S_c^{\mathrm{ref}}
-
\gamma.
\]
Let
\[
\bar S_{\mathrm{ref}}
:=
\E_{\pi_\theta}[S_{\mathrm{ref}}]
=
\sum_{c\in\mathcal A_s}\pi_cS_c^{\mathrm{ref}}.
\]
we have
\begin{align}
\bar r
&=
\sum_{c\in\mathcal A_s}\pi_c r_c
=
\beta\psi'(H)(H-1)
+
\gamma H
-
\gamma\bar S_{\mathrm{ref}}
-
\gamma.
\end{align}
Hence
\begin{align}
r_c-\bar r
&=
\beta\psi'(H)(S_c-H)
+
\gamma(S_c-H)
-
\gamma(S_c^{\mathrm{ref}}-\bar S_{\mathrm{ref}})
\notag\\
&=
\bigl(\beta\psi'(H)+\gamma\bigr)(S_c-H)
-
\gamma(S_c^{\mathrm{ref}}-\bar S_{\mathrm{ref}}).
\label{eq:r_centered}
\end{align}

Substituting \eqref{eq:r_centered} into the second term of
\eqref{eq:param_general_R_drift}, we get
\begin{align}
&\E_{b,c\sim\pi_\theta}
\left[
(S_b-H)(r_c-\bar r)K(b,c;s)
\right]
\notag\\
&\quad
=
\bigl(\beta\psi'(H)+\gamma\bigr)
\E_{b,c\sim\pi_\theta}
\left[
(S_b-H)(S_c-H)K(b,c;s)
\right]
\notag\\
&\qquad
-
\gamma
\E_{b,c\sim\pi_\theta}
\left[
(S_b-H)(S_c^{\mathrm{ref}}-\bar S_{\mathrm{ref}})K(b,c;s)
\right].
\end{align}
Define
\begin{align}
\mathcal V_\theta(S;s)
&:=
\E_{b,c\sim\pi_\theta(\cdot\mid s)}
\left[
(S_b-H)(S_c-H)K(b,c;s)
\right],
\\
\mathcal C_\theta(S,S_{\mathrm{ref}};s)
&:=
\E_{b,c\sim\pi_\theta(\cdot\mid s)}
\left[
(S_b-H)(S_c^{\mathrm{ref}}-\bar S_{\mathrm{ref}})K(b,c;s)
\right].
\end{align}
Then
\begin{align}
D_{\mathrm{RL}}^\theta(a;s)
&=
-A(a,s)
\E_{b\sim\pi_\theta(\cdot\mid s)}
\left[
(H-S_b)K(b,a;s)
\right]
\notag\\
&\quad
+
\bigl(\beta\psi'(H)+\gamma\bigr)
\mathcal V_\theta(S;s)
-
\gamma\mathcal C_\theta(S,S_{\mathrm{ref}};s).
\end{align}

It remains to verify
\[
\mathcal V_\theta(S;s)=\|\nabla_\theta\Hresp(s)\|_2^2.
\]
By \eqref{eq:param_grad_H},
\begin{align}
\|\nabla_\theta\Hresp(s)\|_2^2
&=
\left\langle
\sum_{b}\pi_b(S_b-H)G_b,
\sum_{c}\pi_c(S_c-H)G_c
\right\rangle
\notag\\
&=
\sum_{b,c}
\pi_b\pi_c(S_b-H)(S_c-H)
\langle G_b,G_c\rangle
\notag\\
&=
\mathcal V_\theta(S;s).
\end{align}

Finally, separating the \(b=a\) term from the task-driven part,
\begin{align}
&-A(a,s)
\E_{b\sim\pi_\theta(\cdot\mid s)}
\left[
(H-S_b)K(b,a;s)
\right]
\notag\\
&\quad
=
-A(a,s)
\left[
\pi_\theta(a\mid s)(H-S_a)K(a,a;s)
+
\sum_{b\neq a}\pi_\theta(b\mid s)(H-S_b)K(b,a;s)
\right].
\end{align}
With
\[
B_{\mathrm{ker}}(a;s)
:=
\sum_{b\neq a}
\pi_\theta(b\mid s)(H-S_b)K(b,a;s),
\]
we obtain the split form. This completes the proof. 
\end{proof}

\newpage
\section{Experimental Details}
\subsection{Base RL Methods Used in Experiments}\label{appendix:baselines}
\paragraph{PPO.}
Proximal Policy Optimization (PPO)~\citep{schulman2017ppo} is a representative actor-critic algorithm that stabilizes policy learning by constraining the update to remain close to the behavior policy. In LLM post-training, PPO typically treats each token as an action and estimates token-level advantages with a learned value function, usually via generalized advantage estimation (GAE). Its clipped surrogate objective is
\begin{equation}
    J_{\mathrm{PPO}}(\theta)
=
\mathbb{E}_t\!\left[
\min\!\left(
\rho_t(\theta)\hat A_t,\,
\operatorname{clip}\!\big(\rho_t(\theta),1-\epsilon,1+\epsilon\big)\hat A_t
\right)
\right],
\qquad
\rho_t(\theta)
=
\frac{\pi_\theta(a_t\mid s_t)}{\pi_{\theta_{\mathrm{old}}}(a_t\mid s_t)}.
\end{equation}

PPO is stable and widely adopted, but it is relatively expensive for large language models because it requires an additional critic/value model to estimate $\hat A_t$.

\paragraph{GRPO.}
Group Relative Policy Optimization (GRPO)~\citep{shao2024deepseekmath} extends the group-based idea by replacing critic-based advantages with within-group relative rewards. Given a query $q$, GRPO samples a group of outputs $\{o_i\}_{i=1}^{G}$ and computes a normalized group-based advantage
\begin{equation}
    \hat A_i
=
\frac{R_i-\operatorname{mean}(\{R_j\}_{j=1}^{G})}
{\operatorname{std}(\{R_j\}_{j=1}^{G}) + \epsilon},
\end{equation}
which is shared across all tokens in output $o_i$ under the standard outcome-level setting. The policy is then updated by maximizing the clipped objective
\begin{align}
&J_{\mathrm{GRPO}}(\theta)
=\notag\\
&\mathbb{E}\!\left[
\frac{1}{G}\sum_{i=1}^{G}\frac{1}{|o_i|}\sum_{t=1}^{|o_i|}
\min\!\Big(
\rho_{i,t}(\theta)\hat A_i,\,
\operatorname{clip}\!\big(\rho_{i,t}(\theta),1-\epsilon,1+\epsilon\big)\hat A_i
\Big)
-
\gamma D_{\mathrm{KL}}\!\big(\pi_\theta \,\|\, \pi_{\mathrm{ref}}\big)
\right],
\end{align}
where
\begin{equation}
\rho_{i,t}(\theta)
=
\frac{\pi_\theta(o_{i,t}\mid q,o_{i,<t})}
{\pi_{\theta_{\mathrm{old}}}(o_{i,t}\mid q,o_{i,<t})}.
\end{equation}

GRPO preserves the stable clipped update of PPO while eliminating the critic, which makes it especially attractive for large-scale LLM reinforcement learning.

\paragraph{DAPO.}
Decoupled Clip and Dynamic sAmpling Policy Optimization (DAPO)~\citep{yu2025dapo} is a group-based policy optimization to improve training in long-form reasoning settings, especially long chain-of-thought trajectories. DAPO keeps the group-based advantage formulation
\begin{equation}
    \hat A_i
=
\frac{R_i-\operatorname{mean}(\{R_j\}_{j=1}^{G})}
{\operatorname{std}(\{R_j\}_{j=1}^{G}) + \epsilon},
\end{equation}

but replaces the standard response-level averaging used in GRPO with a token-level aggregation over all tokens in the sampled group, which better balances updates across responses of different lengths:
\begin{align}
    &J_{\mathrm{DAPO}}(\theta)
=\notag\\
&\mathbb{E}\!\left[
\frac{1}{\sum_{i=1}^{G}|o_i|}
\sum_{i=1}^{G}\sum_{t=1}^{|o_i|}
\min\!\Big(
\rho_{i,t}(\theta)\hat A_i,\,
\operatorname{clip}\!\big(\rho_{i,t}(\theta),1-\epsilon_{\mathrm{low}},1+\epsilon_{\mathrm{high}}\big)\hat A_i
\Big)
\right].
\end{align}

In the original formulation, DAPO further removes the explicit KL term and introduces four practical techniques: decoupled asymmetric clipping, dynamic sampling of informative groups, token-level policy-gradient loss, and overlong reward shaping. 

\paragraph{GSPO.}
Group Sequence Policy Optimization (GSPO)~\citep{zheng2025group} is a group-based RL method that moves importance weighting and clipping from token level to sequence level. For a given query \(q\), GSPO samples a group of outputs \(\{o_i\}_{i=1}^{G}\) and uses the same normalized group-based advantage as GRPO:
\begin{equation}
    \hat A_i
=
\frac{R_i-\operatorname{mean}(\{R_j\}_{j=1}^{G})}
{\operatorname{std}(\{R_j\}_{j=1}^{G})}.
\end{equation}
It then defines a length-normalized sequence-level importance ratio
\begin{equation}
s_i(\theta)
=
\left( \frac{\pi_\theta(o_i\mid q)}{\pi_{\theta_{\mathrm{old}}}(o_i\mid q)} \right)^{\frac{1}{|o_i|}}
=
\exp\!\left(
\frac{1}{|o_i|}\sum_{t=1}^{|o_i|}
\log\frac{\pi_\theta(o_{i,t}\mid q,o_{i,<t})}{\pi_{\theta_{\mathrm{old}}}(o_{i,t}\mid q,o_{i,<t})}
\right),
\end{equation}
where the length normalization keeps the ratio scale comparable across responses of different lengths. The policy is optimized with the clipped sequence-level objective
\begin{equation}
J_{\mathrm{GSPO}}(\theta)
=
\mathbb{E}\!\left[
\frac{1}{G}\sum_{i=1}^{G}
\min\!\Big(
s_i(\theta)\hat A_i,\,
\operatorname{clip}\!\big(s_i(\theta),1-\epsilon,1+\epsilon\big)\hat A_i\Big)
\right].
\end{equation}
Compared with token-level clipping, GSPO aligns the optimization granularity with sequence-level rewards and improves training stability.

\subsection{Implementation Details}
\label{appendix:implementation}
We use rule-based outcome rewards across all benchmarks. For ALFWorld and WebShop, successful trajectories receive a reward of $10$, failed trajectories receive $0$, and invalid actions incur an additional penalty of $-0.1$. For the SWE task, we use binary outcome rewards, assigning $1$ to successful trajectories and $0$ otherwise. Across all group-based RL methods, the rollout group size is fixed to $N=8$.

For ALFWorld and WebShop, we use the verl-agent~\citep{fenggroup} training framework. The actor learning rate is set to $1\times 10^{-6}$, the rollout temperature is $1.0$, the validation temperature is $0.4$, and the KL loss coefficient is fixed to $0.01$. We sample $16$ groups per rollout, yielding $128$ environments in total. ALFWorld uses a maximum prompt length of $2048$ tokens and a maximum response length of $512$ tokens, with each episode capped at $50$ environment steps. WebShop uses a maximum prompt length of $4096$ tokens and a maximum response length of $512$ tokens, with each episode capped at $15$ environment steps. We train Qwen2.5-1.5B on $4\times$A800 GPUs and Qwen2.5-7B on $8\times$A800 GPUs for $150$ training steps.

For the SWE task, we use rLLM~\citep{rllm2025} to train Qwen3-32B with a learning rate of $1\times 10^{-6}$. The maximum prompt and response lengths are set to $4096$ and $65536$ tokens, respectively. The training batch size is set to $64$, and we apply rejection sampling with a $2\times$ oversampling ratio: we sample up to twice the target batch size and reject rollout groups whose rewards are all $0$ or all $1$. This increases the proportion of informative samples in each batch while reducing unnecessary forward computation compared with directly training with a batch size of $128$. We train the model on $64\times$H200 GPUs for $250$ steps, with a sampling temperature of $1.0$ during training and $0.6$ during evaluation.

All reported results are averaged over $3$ random seeds. For all AEM experiments, we set $\lambda=1$ and $\epsilon = 10^{-8}$. The temperature $\lambda$ controls the range of the modulation coefficient $\alpha$, while $\epsilon$ is a small stability constant used to prevent numerical instability in min-max normalization and self-calibrated coefficient normalization.
\subsection{Prompts}\label{appendix:prompt}
\begin{promptbox}[title=Prompt Template for ALFWorld]
\begin{lstlisting}[style=mytt]
You are an expert agent operating in the ALFRED embodied Environment. Your task is to: {task_description}. Prior to this step, you have already taken {step_count} step(s). Below are the most recent {history_length} observations and the corresponding actions you took: {action_history}. You are now at step {current_step} and your current observation is: {current_observation}. Your admissible actions of the current situation are: [{admissible_actions}].
Now it's your turn to take an action. You should first reason step-by-step about the current situation. This reasoning process MUST be enclosed within <think> </think> tags. Once you've finished your reasoning, you should choose an admissible action for current step and present it within <action> </action> tags.
\end{lstlisting}
\end{promptbox}

\begin{promptbox}[title=Prompt Template for WebShop]
\begin{lstlisting}[style=mytt]
You are an expert autonomous agent operating in the WebShop e-commerce environment.
Your task is to: {task_description}. Prior to this step, you have already taken {step_count} step(s). Below are the most recent {history_length} observations and the corresponding actions you took: {action_history}. You are now at step {current_step} and your current observation is: {current_observation}. Your admissible actions for the current situation are: [{available_actions}].
Now it's your turn to take one action for the current step. You should first reason step-by-step about the current situation, then think carefully which admissible action best advances the shopping goal. This reasoning process MUST be enclosed within <think> </think> tags.
Once you've finished your reasoning, you should choose an admissible action for current step and present it within <action> </action> tags.
\end{lstlisting}
\end{promptbox}

\begin{promptbox}[title=Prompt Template for R2E Training]
\begin{lstlisting}[style=mytt]
You are a programming agent who is provided a github issue and repository bash environment and is tasked to solve certain tasks (e.g., {task_types}) to resolve the issue.

We have access to the following functions:

- BEGIN FUNCTION #1: file_editor -
Description: {file_editor_description}
Parameters:
{file_editor_parameters}
- END FUNCTION #1 -

- BEGIN FUNCTION #2: execute_bash -
Description: {execute_bash_description}
Parameters:
{execute_bash_parameters}
- END FUNCTION #2 -

- BEGIN FUNCTION #3: search -
Description: {search_description}
Parameters:
{search_parameters}
- END FUNCTION #3 -

- BEGIN FUNCTION #4: finish -
Description: {finish_description}
Parameters:
{finish_parameters}
- END FUNCTION #4 -

If you choose to call a function ONLY reply in the following format with NO suffix:
{function_call_format}

<IMPORTANT>
{important_rules}
</IMPORTANT>
\end{lstlisting}
\end{promptbox}

\begin{promptbox}[title=Prompt Template for SWE-bench-Verified Eval]
\begin{lstlisting}[style=mytt]
You are a programming agent who is provided a github issue and repository bash environment and is tasked to solve certain tasks (e.g., {task_types}) to resolve the issue.

We have access to the following functions:

- BEGIN FUNCTION #1: file_editor -
{file_editor_block}
- END FUNCTION #1 -

- BEGIN FUNCTION #2: execute_bash -
{execute_bash_block}
- END FUNCTION #2 -

- BEGIN FUNCTION #3: search -
{search_block}
- END FUNCTION #3 -

- BEGIN FUNCTION #4: finish -
{finish_block}
- END FUNCTION #4 -

If you choose to call a function ONLY reply in the following format with NO suffix:

{function_call_format}

<IMPORTANT>
{important_rules}
</IMPORTANT>
\end{lstlisting}
\end{promptbox}

\newpage

\section*{NeurIPS Paper Checklist}

\begin{enumerate}

\item {\bf Claims}
    \item[] Question: Do the main claims made in the abstract and introduction accurately reflect the paper's contributions and scope?
    \item[] Answer: \answerYes{} 
    \item[] Justification: The main claims in the abstract and introduction accurately reflect the paper's contributions and scope: to address the sparsity of credit assignment in agentic RL, we develop a response-level entropy analysis framework, propose AEM, a supervision-free, lightweight, plug-in method, and validate across multiple model scales and strong baselines that it consistently improves RL performance.

    \item[] Guidelines:
    \begin{itemize}
        \item The answer \answerNA{} means that the abstract and introduction do not include the claims made in the paper.
        \item The abstract and/or introduction should clearly state the claims made, including the contributions made in the paper and important assumptions and limitations. A \answerNo{} or \answerNA{} answer to this question will not be perceived well by the reviewers. 
        \item The claims made should match theoretical and experimental results, and reflect how much the results can be expected to generalize to other settings. 
        \item It is fine to include aspirational goals as motivation as long as it is clear that these goals are not attained by the paper. 
    \end{itemize}

\item {\bf Limitations}
    \item[] Question: Does the paper discuss the limitations of the work performed by the authors?
    \item[] Answer: \answerYes{} 
    \item[] Justification: See Appendix~\ref{appendix:limitations} for limitations.
    \item[] Guidelines:
    \begin{itemize}
        \item The answer \answerNA{} means that the paper has no limitation while the answer \answerNo{} means that the paper has limitations, but those are not discussed in the paper. 
        \item The authors are encouraged to create a separate ``Limitations'' section in their paper.
        \item The paper should point out any strong assumptions and how robust the results are to violations of these assumptions (e.g., independence assumptions, noiseless settings, model well-specification, asymptotic approximations only holding locally). The authors should reflect on how these assumptions might be violated in practice and what the implications would be.
        \item The authors should reflect on the scope of the claims made, e.g., if the approach was only tested on a few datasets or with a few runs. In general, empirical results often depend on implicit assumptions, which should be articulated.
        \item The authors should reflect on the factors that influence the performance of the approach. For example, a facial recognition algorithm may perform poorly when image resolution is low or images are taken in low lighting. Or a speech-to-text system might not be used reliably to provide closed captions for online lectures because it fails to handle technical jargon.
        \item The authors should discuss the computational efficiency of the proposed algorithms and how they scale with dataset size.
        \item If applicable, the authors should discuss possible limitations of their approach to address problems of privacy and fairness.
        \item While the authors might fear that complete honesty about limitations might be used by reviewers as grounds for rejection, a worse outcome might be that reviewers discover limitations that aren't acknowledged in the paper. The authors should use their best judgment and recognize that individual actions in favor of transparency play an important role in developing norms that preserve the integrity of the community. Reviewers will be specifically instructed to not penalize honesty concerning limitations.
    \end{itemize}

\item {\bf Theory assumptions and proofs}
    \item[] Question: For each theoretical result, does the paper provide the full set of assumptions and a complete (and correct) proof?
    \item[] Answer: \answerYes{} 
    \item[] Justification: See Section~\ref{sec:theoretical} for theoretical results and Appendix~\ref{sec:proofs} for detailed proofs.
    \item[] Guidelines:
    \begin{itemize}
        \item The answer \answerNA{} means that the paper does not include theoretical results. 
        \item All the theorems, formulas, and proofs in the paper should be numbered and cross-referenced.
        \item All assumptions should be clearly stated or referenced in the statement of any theorems.
        \item The proofs can either appear in the main paper or the supplemental material, but if they appear in the supplemental material, the authors are encouraged to provide a short proof sketch to provide intuition. 
        \item Inversely, any informal proof provided in the core of the paper should be complemented by formal proofs provided in appendix or supplemental material.
        \item Theorems and Lemmas that the proof relies upon should be properly referenced. 
    \end{itemize}

    \item {\bf Experimental result reproducibility}
    \item[] Question: Does the paper fully disclose all the information needed to reproduce the main experimental results of the paper to the extent that it affects the main claims and/or conclusions of the paper (regardless of whether the code and data are provided or not)?
    \item[] Answer: \answerYes{} 
    \item[] Justification:  We provide the pesudo-code of AEM in Appendix~\ref{appendix:AEM} and detailed  training implementation and prompt in Appendix~\ref{appendix:implementation} and Appendix~\ref{appendix:prompt}, respectively.
    \item[] Guidelines:
    \begin{itemize}
        \item The answer \answerNA{} means that the paper does not include experiments.
        \item If the paper includes experiments, a \answerNo{} answer to this question will not be perceived well by the reviewers: Making the paper reproducible is important, regardless of whether the code and data are provided or not.
        \item If the contribution is a dataset and\slash or model, the authors should describe the steps taken to make their results reproducible or verifiable. 
        \item Depending on the contribution, reproducibility can be accomplished in various ways. For example, if the contribution is a novel architecture, describing the architecture fully might suffice, or if the contribution is a specific model and empirical evaluation, it may be necessary to either make it possible for others to replicate the model with the same dataset, or provide access to the model. In general. releasing code and data is often one good way to accomplish this, but reproducibility can also be provided via detailed instructions for how to replicate the results, access to a hosted model (e.g., in the case of a large language model), releasing of a model checkpoint, or other means that are appropriate to the research performed.
        \item While NeurIPS does not require releasing code, the conference does require all submissions to provide some reasonable avenue for reproducibility, which may depend on the nature of the contribution. For example
        \begin{enumerate}
            \item If the contribution is primarily a new algorithm, the paper should make it clear how to reproduce that algorithm.
            \item If the contribution is primarily a new model architecture, the paper should describe the architecture clearly and fully.
            \item If the contribution is a new model (e.g., a large language model), then there should either be a way to access this model for reproducing the results or a way to reproduce the model (e.g., with an open-source dataset or instructions for how to construct the dataset).
            \item We recognize that reproducibility may be tricky in some cases, in which case authors are welcome to describe the particular way they provide for reproducibility. In the case of closed-source models, it may be that access to the model is limited in some way (e.g., to registered users), but it should be possible for other researchers to have some path to reproducing or verifying the results.
        \end{enumerate}
    \end{itemize}

\item {\bf Open access to data and code}
    \item[] Question: Does the paper provide open access to the data and code, with sufficient instructions to faithfully reproduce the main experimental results, as described in supplemental material?
    \item[] Answer: \answerNo{} 
    \item[] Justification: All datasets and backbone models are public. Since AEM is a lightweight plug-in algorithm whose implementation depends on the underlying RL method, AEM itself introduces little additional engineering complexity. Moreover,  pseudo-code Algorithm~\ref{alg:AEM} plus implementation Appendix~\ref{appendix:implementation} and prompt ~\ref{appendix:prompt} provide sufficient detail for reproduction.
    \item[] Guidelines:
    \begin{itemize}
        \item The answer \answerNA{} means that paper does not include experiments requiring code.
        \item Please see the NeurIPS code and data submission guidelines (\url{https://neurips.cc/public/guides/CodeSubmissionPolicy}) for more details.
        \item While we encourage the release of code and data, we understand that this might not be possible, so \answerNo{} is an acceptable answer. Papers cannot be rejected simply for not including code, unless this is central to the contribution (e.g., for a new open-source benchmark).
        \item The instructions should contain the exact command and environment needed to run to reproduce the results. See the NeurIPS code and data submission guidelines (\url{https://neurips.cc/public/guides/CodeSubmissionPolicy}) for more details.
        \item The authors should provide instructions on data access and preparation, including how to access the raw data, preprocessed data, intermediate data, and generated data, etc.
        \item The authors should provide scripts to reproduce all experimental results for the new proposed method and baselines. If only a subset of experiments are reproducible, they should state which ones are omitted from the script and why.
        \item At submission time, to preserve anonymity, the authors should release anonymized versions (if applicable).
        \item Providing as much information as possible in supplemental material (appended to the paper) is recommended, but including URLs to data and code is permitted.
    \end{itemize}

\item {\bf Experimental setting/details}
    \item[] Question: Does the paper specify all the training and test details (e.g., data splits, hyperparameters, how they were chosen, type of optimizer) necessary to understand the results?
    \item[] Answer: \answerYes{} 
    \item[] Justification: Section~\ref{sec:setup} describes the experimental setup, including the benchmarks and evaluation settings. Training curves are provided in the appendix~\ref{appendix:curves}, and detailed implementation details and prompts, are included in~\ref{appendix:implementation} and ~\ref{appendix:prompt}, respectively.
    \item[] Guidelines:
    \begin{itemize}
        \item The answer \answerNA{} means that the paper does not include experiments.
        \item The experimental setting should be presented in the core of the paper to a level of detail that is necessary to appreciate the results and make sense of them.
        \item The full details can be provided either with the code, in appendix, or as supplemental material.
    \end{itemize}

\item {\bf Experiment statistical significance}
    \item[] Question: Does the paper report error bars suitably and correctly defined or other appropriate information about the statistical significance of the experiments?
    \item[] Answer: \answerYes{} 
    \item[] Justification: We report both the mean and standard deviation (std) across 3 independent
runs with different random seeds for all key results Table~\ref{tab:alfworld_webshop}, Table~\ref{tab:swebench_main}, and~\ref{tab:webshop_qwen15b_horizontal}.
    \item[] Guidelines:
    \begin{itemize}
        \item The answer \answerNA{} means that the paper does not include experiments.
        \item The authors should answer \answerYes{} if the results are accompanied by error bars, confidence intervals, or statistical significance tests, at least for the experiments that support the main claims of the paper.
        \item The factors of variability that the error bars are capturing should be clearly stated (for example, train/test split, initialization, random drawing of some parameter, or overall run with given experimental conditions).
        \item The method for calculating the error bars should be explained (closed form formula, call to a library function, bootstrap, etc.)
        \item The assumptions made should be given (e.g., Normally distributed errors).
        \item It should be clear whether the error bar is the standard deviation or the standard error of the mean.
        \item It is OK to report 1-sigma error bars, but one should state it. The authors should preferably report a 2-sigma error bar than state that they have a 96\% CI, if the hypothesis of Normality of errors is not verified.
        \item For asymmetric distributions, the authors should be careful not to show in tables or figures symmetric error bars that would yield results that are out of range (e.g., negative error rates).
        \item If error bars are reported in tables or plots, the authors should explain in the text how they were calculated and reference the corresponding figures or tables in the text.
    \end{itemize}

\item {\bf Experiments compute resources}
    \item[] Question: For each experiment, does the paper provide sufficient information on the computer resources (type of compute workers, memory, time of execution) needed to reproduce the experiments?
    \item[] Answer: \answerYes{}{} 
    \item[] Justification: See Appendix~\ref{appendix:implementation}.
    \item[] Guidelines:
    \begin{itemize}
        \item The answer \answerNA{} means that the paper does not include experiments.
        \item The paper should indicate the type of compute workers CPU or GPU, internal cluster, or cloud provider, including relevant memory and storage.
        \item The paper should provide the amount of compute required for each of the individual experimental runs as well as estimate the total compute. 
        \item The paper should disclose whether the full research project required more compute than the experiments reported in the paper (e.g., preliminary or failed experiments that didn't make it into the paper). 
    \end{itemize}
    
\item {\bf Code of ethics}
    \item[] Question: Does the research conducted in the paper conform, in every respect, with the NeurIPS Code of Ethics \url{https://neurips.cc/public/EthicsGuidelines}?
    \item[] Answer: \answerYes{} 
    \item[] Justification: Our research fully complies with the NeurIPS Code of Ethics.
    \item[] Guidelines:
    \begin{itemize}
        \item The answer \answerNA{} means that the authors have not reviewed the NeurIPS Code of Ethics.
        \item If the authors answer \answerNo, they should explain the special circumstances that require a deviation from the Code of Ethics.
        \item The authors should make sure to preserve anonymity (e.g., if there is a special consideration due to laws or regulations in their jurisdiction).
    \end{itemize}

\item {\bf Broader impacts}
    \item[] Question: Does the paper discuss both potential positive societal impacts and negative societal impacts of the work performed?
    \item[] Answer: \answerYes{} 
    \item[] Justification: See Appendix~\ref{appendix:broader}.
    \item[] Guidelines:
    \begin{itemize}
        \item The answer \answerNA{} means that there is no societal impact of the work performed.
        \item If the authors answer \answerNA{} or \answerNo, they should explain why their work has no societal impact or why the paper does not address societal impact.
        \item Examples of negative societal impacts include potential malicious or unintended uses (e.g., disinformation, generating fake profiles, surveillance), fairness considerations (e.g., deployment of technologies that could make decisions that unfairly impact specific groups), privacy considerations, and security considerations.
        \item The conference expects that many papers will be foundational research and not tied to particular applications, let alone deployments. However, if there is a direct path to any negative applications, the authors should point it out. For example, it is legitimate to point out that an improvement in the quality of generative models could be used to generate Deepfakes for disinformation. On the other hand, it is not needed to point out that a generic algorithm for optimizing neural networks could enable people to train models that generate Deepfakes faster.
        \item The authors should consider possible harms that could arise when the technology is being used as intended and functioning correctly, harms that could arise when the technology is being used as intended but gives incorrect results, and harms following from (intentional or unintentional) misuse of the technology.
        \item If there are negative societal impacts, the authors could also discuss possible mitigation strategies (e.g., gated release of models, providing defenses in addition to attacks, mechanisms for monitoring misuse, mechanisms to monitor how a system learns from feedback over time, improving the efficiency and accessibility of ML).
    \end{itemize}
    
\item {\bf Safeguards}
    \item[] Question: Does the paper describe safeguards that have been put in place for responsible release of data or models that have a high risk for misuse (e.g., pre-trained language models, image generators, or scraped datasets)?
    \item[] Answer: \answerNA{} 
    \item[] Justification: This paper poses no such risks.
    \item[] Guidelines:
    \begin{itemize}
        \item The answer \answerNA{} means that the paper poses no such risks.
        \item Released models that have a high risk for misuse or dual-use should be released with necessary safeguards to allow for controlled use of the model, for example by requiring that users adhere to usage guidelines or restrictions to access the model or implementing safety filters. 
        \item Datasets that have been scraped from the Internet could pose safety risks. The authors should describe how they avoided releasing unsafe images.
        \item We recognize that providing effective safeguards is challenging, and many papers do not require this, but we encourage authors to take this into account and make a best faith effort.
    \end{itemize}

\item {\bf Licenses for existing assets}
    \item[] Question: Are the creators or original owners of assets (e.g., code, data, models), used in the paper, properly credited and are the license and terms of use explicitly mentioned and properly respected?
    \item[] Answer: \answerYes{} 
    \item[] Justification: We properly cite all benchmarks (e.g., ALFWorld, WebShop), pretrained models (e.g., Qwen2.5), and APIs (e.g., GPT-5.2-Pro and Gemini-3-Pro) with acknowledgment of their licenses and terms of use.
    \item[] Guidelines:
    \begin{itemize}
        \item The answer \answerNA{} means that the paper does not use existing assets.
        \item The authors should cite the original paper that produced the code package or dataset.
        \item The authors should state which version of the asset is used and, if possible, include a URL.
        \item The name of the license (e.g., CC-BY 4.0) should be included for each asset.
        \item For scraped data from a particular source (e.g., website), the copyright and terms of service of that source should be provided.
        \item If assets are released, the license, copyright information, and terms of use in the package should be provided. For popular datasets, \url{paperswithcode.com/datasets} has curated licenses for some datasets. Their licensing guide can help determine the license of a dataset.
        \item For existing datasets that are re-packaged, both the original license and the license of the derived asset (if it has changed) should be provided.
        \item If this information is not available online, the authors are encouraged to reach out to the asset's creators.
    \end{itemize}

\item {\bf New assets}
    \item[] Question: Are new assets introduced in the paper well documented and is the documentation provided alongside the assets?
    \item[] Answer: \answerNA{} 
    \item[] Justification: The paper does not introduce or release new assets such as datasets, models, or code at submission time. Our work proposes a plug-in training method that is evaluated using existing public datasets and backbone models.
    \item[] Guidelines:
    \begin{itemize}
        \item The answer \answerNA{} means that the paper does not release new assets.
        \item Researchers should communicate the details of the dataset\slash code\slash model as part of their submissions via structured templates. This includes details about training, license, limitations, etc. 
        \item The paper should discuss whether and how consent was obtained from people whose asset is used.
        \item At submission time, remember to anonymize your assets (if applicable). You can either create an anonymized URL or include an anonymized zip file.
    \end{itemize}

\item {\bf Crowdsourcing and research with human subjects}
    \item[] Question: For crowdsourcing experiments and research with human subjects, does the paper include the full text of instructions given to participants and screenshots, if applicable, as well as details about compensation (if any)? 
    \item[] Answer: \answerNA{} 
    \item[] Justification: No human subjects were involved in this work.
    \item[] Guidelines:
    \begin{itemize}
        \item The answer \answerNA{} means that the paper does not involve crowdsourcing nor research with human subjects.
        \item Including this information in the supplemental material is fine, but if the main contribution of the paper involves human subjects, then as much detail as possible should be included in the main paper. 
        \item According to the NeurIPS Code of Ethics, workers involved in data collection, curation, or other labor should be paid at least the minimum wage in the country of the data collector. 
    \end{itemize}

\item {\bf Institutional review board (IRB) approvals or equivalent for research with human subjects}
    \item[] Question: Does the paper describe potential risks incurred by study participants, whether such risks were disclosed to the subjects, and whether Institutional Review Board (IRB) approvals (or an equivalent approval/review based on the requirements of your country or institution) were obtained?
    \item[] Answer: \answerNA{} 
    \item[] Justification: No human subjects were involved in this work.
    \item[] Guidelines:
    \begin{itemize}
        \item The answer \answerNA{} means that the paper does not involve crowdsourcing nor research with human subjects.
        \item Depending on the country in which research is conducted, IRB approval (or equivalent) may be required for any human subjects research. If you obtained IRB approval, you should clearly state this in the paper. 
        \item We recognize that the procedures for this may vary significantly between institutions and locations, and we expect authors to adhere to the NeurIPS Code of Ethics and the guidelines for their institution. 
        \item For initial submissions, do not include any information that would break anonymity (if applicable), such as the institution conducting the review.
    \end{itemize}

\item {\bf Declaration of LLM usage}
    \item[] Question: Does the paper describe the usage of LLMs if it is an important, original, or non-standard component of the core methods in this research? Note that if the LLM is used only for writing, editing, or formatting purposes and does \emph{not} impact the core methodology, scientific rigor, or originality of the research, declaration is not required.
    \item[] Answer: \answerNA{} 
    \item[] Justification: The core methodology of this paper is fully original and developed by the
authors without the use of LLMs.
    \item[] Guidelines:
    \begin{itemize}
        \item The answer \answerNA{} means that the core method development in this research does not involve LLMs as any important, original, or non-standard components.
        \item Please refer to our LLM policy in the NeurIPS handbook for what should or should not be described.
    \end{itemize}

\end{enumerate}

\end{document}